\documentclass[]{article}

\usepackage{mathrsfs,bm,color}

\newcommand{\para}[1]{\left( {#1} \right)}

\newcommand{\myparagraph}[1]{\noindent\textbf{#1}}

\newcommand{\vect}[1]{\ensuremath{\boldsymbol{\mathbf{#1}}}}
\newcommand{\rdmvect}[1]{\ensuremath{\bm{#1}}}
\newcommand{\mat}[1]{\ensuremath{\boldsymbol{\mathbf{#1}}}}
\newcommand{\rdmmat}[1]{\ensuremath{\bm{#1}}}
\newcommand{\tens}[1]{ {\ensuremath{\underline{\boldsymbol{\mathbf{#1}}}}}}
\newcommand{\rdmtens}[1]{ {\ensuremath{\underline{\bm{#1}}}}}
\newcommand{\norm}[1]{\ensuremath{ \|#1 \|}}
\newcommand{\R}{\mathbb{R}}
\newcommand{\setim}{\mathcal{I}_m}
\newcommand{\setit}{\mathcal{I}_t}
\newcommand{\ind}{\mathds{1}}
\newcommand{\rdm}[1]{  \mathrm{#1}  }

\usepackage[utf8]{inputenc} 
\usepackage[T1]{fontenc}   
\usepackage[margin=1in]{geometry}
\usepackage{hyperref}       
\hypersetup{
    colorlinks = true,
    citecolor = blue,
    linkcolor = red
}

\usepackage{url}            
\usepackage{booktabs}       
\usepackage{amsfonts}     
\usepackage{nicefrac}       
\usepackage{microtype}      
\usepackage{xcolor}         
\usepackage{color}
\usepackage{float}
\usepackage[style=base]{caption}
\usepackage{subcaption}
\usepackage{tikz}
\usepackage{tikz-cd}
\usepackage{pgfplots}
\usepackage{abstract}
\usepgfplotslibrary{groupplots}
\pgfplotsset{compat=1.18}

\usepackage{amsmath,amssymb,amsthm,mathtools,bbm}
\usepackage{physics}
\usepackage{bbold}
\usepackage{dsfont}
\usepackage{mathrsfs}
\usepackage[ruled,vlined,linesnumbered]{algorithm2e}
\usepackage[capitalize]{cleveref}
\usepackage{xfrac}
\usepackage{graphicx}
\usepackage{authblk}
\usepackage[numbers]{natbib}

\usepackage{times}
\usepackage{comment}

\usepackage{enumitem}
\setlist[enumerate]{leftmargin=1.5em}
\setlist[itemize]{leftmargin=1.5em}
\usepackage{autonum}

\usepackage{stmaryrd}       
\usepackage{algpseudocode}  
\usepackage[normalem]{ulem}  
\usepackage{changepage}    
\usepackage{bm}  

\numberwithin{equation}{section}
\usepackage{thmtools}
\usepackage{mdframed} 

\declaretheorem[thmbox=M,name=Theorem,numberwithin=section]{mytheo}
\declaretheorem[thmbox=M,name=Proposition,numberwithin=section]{myprop}
\declaretheorem[thmbox=S,name=Lemma,numberwithin=section]{mylemma}

\theoremstyle{definition}
\newtheorem{mydef}{Definition}[section]

\makeatletter
\renewenvironment{proof}[1][\proofname]{\par
  \pushQED{\qed}%
  \normalfont \topsep6\p@\@plus6\p@\relax
  \trivlist
  \item[\hskip\labelsep\textbf{#1.}]\ignorespaces
}{%
  \popQED\endtrivlist\@endpefalse
}
\makeatother

\graphicspath{{imgs/}}
\begin{document}

\title{Computational Thresholds in Multi-Modal Learning \\ via the Spiked Matrix-Tensor Model}

\author[1]{Hugo Tabanelli}
\author[1]{Pierre Mergny}
\author[2]{Lenka Zdeborov\'a}
\author[1]{Florent Krzakala}

\affil[1]{Information Learning and Physics Laboratory, \'Ecole Polytechnique F\'ed\'erale de Lausanne (EPFL)}
\affil[2]{Statistical Physics of Computation Laboratory, \'Ecole Polytechnique F\'ed\'erale de Lausanne (EPFL)}

\maketitle

\begin{abstract}
We study the recovery of multiple high-dimensional signals from two noisy, correlated modalities: a spiked matrix and a spiked tensor sharing a common low-rank structure. This setting generalizes classical spiked matrix and tensor models, unveiling intricate interactions between inference channels and surprising algorithmic behaviors. Notably, while the spiked tensor model is typically intractable at low signal-to-noise ratios, its correlation with the matrix enables efficient recovery via Bayesian Approximate Message Passing, inducing staircase-like phase transitions reminiscent of neural network phenomena. In contrast, empirical risk minimization for joint learning fails: the tensor component obstructs effective matrix recovery, and joint optimization significantly degrades performance, highlighting the limitations of naive multi-modal learning. We show that a simple Sequential Curriculum Learning strategy—first recovering the matrix, then leveraging it to guide tensor recovery—resolves this bottleneck and achieves optimal weak recovery thresholds. This strategy, implementable with spectral methods, emphasizes the critical role of structural correlation and learning order in multi-modal high-dimensional inference.
\end{abstract}

\section{Introduction}

Understanding which learning problems are inherently hard and which are easy is a fundamental question in high-dimensional inference and machine learning theory. While convex optimization is well understood, recent theoretical efforts have focused on structured non-convex functions to explain the practical efficiency of training neural networks \cite{lecun2015deep}. This question is particularly interesting in the context of multi-modal learning \citep{ngiam2011multimodal,baltruvsaitis2018multimodal}, where different but correlated data sources/modalities provide complementary views of the same latent structure. A central issue in this setting is how integrating information from multiple modalities improves statistical power and aligns shared structures, compared to analyzing each modality independently. This raises important questions about the design of the optimization process — for example, in what order to incorporate modalities, whether to use a curriculum \citep{bengio2009curriculum}, and how such choices affect convergence and generalization.

Here, we consider this question in the framework of a popular line of theoretical research on the learnability of spiked matrix and tensor models (see, e.g., \citep{arous2005phase,donoho2018optimal,richard2014statistical,lesieur2017statistical}). Retrieving information hidden in a spiked matrix can be done efficiently—and optimally—using a simple spectral method. In contrast, retrieving information optimally from a spiked tensor is significantly harder and gives rise to the well-studied statistical-to-computational gap at the root of recent efforts to understand computational hardness \citep{hopkins2015tensor,perry_statistical_2020,wein2019kikuchi,arous2020algorithmic}.

\vspace{2mm}
\myparagraph{The matrix-tensor model ---} Motivated by these considerations, we introduce and study a mixed spiked matrix-tensor model where correlated data modalities provide complementary information on joint latent variables. The goal is to recover the “hidden” latent variables $\vect{u}^\star \in \R^{N_1}$, $\vect{v}^\star \in \R^{N_2}$, $\vect{x}^\star \in \R^{N_3}$, and $\vect{y}^\star \in \R^{N_4}$ given the knowledge of a spiked  matrix $\rdmmat{Y}_m \in \R^{N_1 \times N_2}$ and tensor $\rdmtens{Y}_t \in \R^{N_1 \times N_3 \times N_4}$ defined component-wise as:
\begin{align}
     \rdmmat{Y}_{m} = \sqrt{\Delta_m}\, \rdmmat{Z} + \frac{1}{\sqrt{N_1}} \, \vect{u}^\star \para{\vect{v}^\star}^T & \hspace{0.7cm}\textrm{or equivalently}& Y_{m,\, i j} = \sqrt{\Delta_m}\, Z_{i j} + \frac{1}{\sqrt{N_1}} \, u^\star_{i}v^\star_{j}  \label{eq:uts_setting_m} \,  \, , \\
    \rdmtens{Y}_{t} = \sqrt{\Delta_t} \, \rdmtens{Z} + \frac{1}{N_1}\vect{u}^\star\otimes \vect{x}^\star \otimes \vect{y}^\star &\hspace{0.7cm} \textrm{or equivalently}&   Y_{t,\, i j k} = \sqrt{\Delta_t} \, Z_{i j k} + \frac{1}{N_1}u^\star_{i} x^\star_{j} y^\star_{k} \label{eq:uts_setting_t} \, ,
\end{align}
where $\rdmmat{Z} \in \R^{N_1 \times N_2}$ and $\rdmtens{Z} \in \R^{N_1 \times N_3 \times N_4}$ are Gaussian noise matrix and tensor with independent and identically distributed (i.i.d.) entries, scaled by noise parameters $\Delta_m$ and $\Delta_t$ for the matrix and tensor channels, respectively. Alternatively, one can note the signal-to-noise ratios are defined as $\lambda_2 = \frac{1}{\Delta_m}$ and $\lambda_3 = \frac{1}{\Delta_t}$.  Taken individually, these problems reduce to a spike Wishart model \cite{donoho2018optimal} for the matrix, and a non-symmetric spiked tensor model \citep{barbier2017layered,kadmon2018statistical}. The crucial ingredient in the matrix-tensor model is that one of the latent vectors in $\rdmmat{Y}_m$ (the vector ${\bf u}^\star$) is also appearing in the tensor $\rdmtens{Y}_t$.

We can interpret this as a basic model for multi-modal learning in two ways. In one interpretation, $N_1$ and $N_4$ represent the data dimensionalities of two modalities, while $N_2$ and $N_3$ correspond to the number of samples from each modality—for instance, a medical condition represented by a latent vector $\vect{u}^\star$ that is observed through two different imaging techniques. In an alternative interpretation, $N_1$ denotes the number of samples, and $N_2$, $N_3$, and $N_4$ represent the dimensionalities of various views of each sample—such as images and their corresponding captions, with $u_i^\star$ indicating the underlying topic of sample $i$. Since we are not focused on any particular application, we will treat the matrix-tensor model abstractly, encompassing both perspectives.

To provide rigorous analytical insight into the learnability of the hidden vectors and on the interplay of learning both from tensor and matrix modalities, we consider the high dimensional asymptotic limit  where $N_1, N_2, N_3, N_4 \to \infty$, but with fixed  ratios  defined as:
\begin{align}
    \alpha_{2} := \lim_{N \to \infty} \frac{N_2}{N_1}, \qquad \alpha_{3} := \lim_{N \to \infty} \frac{N_3}{N_1}, \qquad \alpha_{4} := \lim_{N \to \infty} \frac{N_4}{N_1}.
\end{align}
The signals $\vect{u}^\star$, $\vect{v}^\star$, $\vect{x}^\star$, and $\vect{y}^\star$ are latent vectors drawn from a Gaussian prior, normalized to ensure a norm of $\sqrt{N}$ with $N$ the size of the corresponding vector. Equivalently, they can be taken on a hyper-sphere with radius $\sqrt{N}$.

\vspace{2mm}
\myparagraph{Main results ---}
This model reveals both the benefits and challenges of learning from multiple modalities.
We now summarize our main contributions:
\vspace{1mm}
\begin{enumerate}[label=\roman*),
  left=0pt,
  labelsep=0.5em,
  itemsep=1mm,
  parsep=0pt,
  topsep=0pt,
  partopsep=0pt
]
   \item Our first finding is that the correlated structure drastically affects the computational complexity of the problem. In particular, we show that one can efficiently recover the tensor through a coupling mechanism \cite{abbe2021staircase}, where recovery of the spike ${\bf u}^\star$ jump-starts the learning of the tensor. This resembles a staircase effect, similar to that observed in multi-index models \cite{abbe2021staircase,abbe2023sgd,troiani_fundamental_2024}, and significantly improves computational feasibility. This illustrates how multi-modality can help by providing auxiliary information that simplifies difficult directions and facilitates their learning.      We prove this learnability result by explicitly analyzing the performance of the most efficient known iterative algorithm for such problems: the Bayesian version of Approximate Message Passing (Bayes-AMP) \citep{celentano2021high,deshpande2014information,lesieur2017statistical}. We present a detailed phase diagram and compute the sharp phase transitions arising in this setting. Specifically, we show that the weak recovery thresholds are $\Delta_m = \sqrt{\alpha_2}$ for the matrix, and $\Delta_t = \sqrt{\alpha_3 \alpha_4} \frac{\alpha_2 - \Delta_m^2}{\alpha_2 + \Delta_m}$ for the tensor.
    
   \item While the previous result highlights the potential computational advantages of multi-modal learning, Bayesian optimal algorithms are rarely used in practice. We thus turn to a more common—and arguably more relevant—setting based on empirical risk minimization (ERM), analyzed using a variant of Approximate Message Passing that mimics ERM and serves as a proxy for gradient descent-style optimization. This approach allows us to illustrate the pitfalls that can arise when optimizing over heterogeneous structures.
    In particular, we show that the recovery thresholds can differ significantly from the Bayesian case, and we again provide explicit expressions. Crucially, {\it our analysis reveals that joint (simultaneous) risk minimization over both the tensor and the matrix significantly worsens recovery performance}: even though the tensor—strictly speaking—provides more information, joint optimization makes the matrix estimation (an otherwise easy problem) much harder, leading to a substantial shift in the weak recovery threshold. 
    
     \item In constrast, we show that if the matrix is estimated first and this estimate is then used to initialize the tensor learning, the optimal Bayesian thresholds for weak recovery are recovered in the ERM setting. This underscores the importance of sequential learning in such problems. This can be effectively implemented using a simple off-the-shelf spectral method where  we perform principal component analysis (PCA) first on the matrix, and then on the tensor.
\end{enumerate}

\vspace{2mm}
These results provide a rare instance where the interaction between different data modalities can be precisely analyzed, shedding light on how structure and learning order affect computational tractability in high-dimensional problems. By bridging the gap between statistical and algorithmic perspectives across both Bayesian and empirical risk minimization settings, our work clarifies the benefits of multi-modal integration and highlights the potential dangers of naive joint optimization. It also underscores the potential importance of a curriculum, where learning one part of the problem first enables the successful learning of harder components later—but not simultaneously. We believe these findings offer a useful perspective on structured non-convex optimization and may inform the design of more effective strategies in multi-modal machine learning settings.

\vspace{2mm}
\myparagraph{Related works ---}

{\bf Spiked models} have attracted significant attention over the past decades, both in their matrix and tensor formulations. The matrix version—where a fixed-rank signal matrix is added to a random Wigner or Wishart matrix—has been foundational in the study of principal component analysis (PCA) in the high-dimensional regime. In this setting, the optimal weak recovery threshold can be achieved efficiently using spectral methods \cite{donoho2018optimal,baik_phase_2005}. A key result is that the top eigenvalue of the perturbed random matrix undergoes a sharp phase transition at the recovery threshold: below this point, it remains within the bulk of the spectrum and the top eigenvector is uninformative (uniformly distributed on the sphere), while above it, the eigenvalue separates and the eigenvector becomes partially aligned with the underlying signal.

In contrast, the {\bf tensor version} is found to be particularly challenging, as its statistical and computational thresholds scale differently \citep{hopkins2015tensor,perry_statistical_2020,lesieur2017statistical,Lesieur_2017,wein2019kikuchi,arous2020algorithmic}: the noise threshold $\Delta = 1/\sqrt{\lambda}$ is of order $O(1)$ for statistical recovery and $O(N^{-1/2})$ for computational recovery, making signal recovery a complex and intricate task. A pure version of the matrix-tensor model has been studied in physics under the name of the mixed p-spin model \cite{crisanti2004spherical, chen2017parisi}.

{\bf Mixed spiked matrix-tensor} models were studied previously in the symmetric version where ${\bf u}^\star = {\bf v}^\star = {\bf x}^\star = {\bf y}^\star$ \cite{Sarao_Mannelli_marvels,mannelli_passed_2019,sarao_mannelli_who_2019}. In this version, the staircase phenomenon, where having learned a part of the signal helps learning the rest -- a property which is behind the findings of the present paper -- does not appear. Moreover, the symmetric version is not suitable for the interpretation in terms of multi-modal learning because there is then no distinction between the dimensionality and the number of samples. 

{\bf Approximate Message Passing} has become a tool of choice for the rigorous analysis of high-dimensional statistical inference problems and algorithms. It provides sharp asymptotic characterizations of both Bayesian and empirical risk minimization procedures in models such as spiked matrices and tensors \citep{donoho_message-passing_2009,javanmard_state_2013,richard2014statistical,lesieur2017statistical,celentano2021high,alaoui2023sampling}. In this work, we use AMP not only to derive theoretical performance guarantees, but also to uncover the algorithmic mechanisms—such as coupling and sequential learning—that govern multi-modal inference.

{\bf The staircase effect} has recently attracted significant attention in the context of learning multi-index models with neural networks \cite{abbe2022merged, abbe2023sgd,troiani_fundamental_2024}, as it significantly improves computational feasibility and is believed to play a fundamental role in how neural networks learn complex functions \cite{abbe2021staircase}. The coupling we observe between the learning of the matrix and the tensor in our setting can be seen as an analog of staircase learning for tensor models.

{\bf Multi-modal learning} in current machine learning focuses on learning complex, non-linear models for each modality that ideally cross-inform and enhance each other \citep{ngiam2011multimodal,baltruvsaitis2018multimodal,bayoudh2022survey}. There are multiple theoretical investigations of questions related to multi-modal learning in solvable high-dimensional models. E.g.  \cite{abdelaleem2023simultaneous,keup2024optimal} studied a model related to ours where two matrices with correlated spikes were observed, however, this model does not present such a rich interplay between statistical and computational aspects as arise in the model we consider here due to the tensor part. 
Authors of \cite{nandy2024multimodal} considered yet another model where two matrices are observed and approximate message passing is analyzed.

\section{Definitions}\label{sec:gen}

\myparagraph{Notations ---} Generic vectors $\vect{x}$ are written in bold lowercase. In the following, we will use italic font to highlight random vectors conditioned on $\vect{x}$  (e.g. $\rdmvect{y}$). For two vectors $\vect{a}, \vect{b} \in \mathbb{R}^d$, we denote by $\langle \vect{a} , \vect{b} \rangle := \sum_{i=1}^{d} a_i b_i$ their scalar product, $\| \vect{a} \| := \sqrt{\sum_{i=1}^d a_i^2}$ will always correspond to the $\ell^2$-norm of $\vect{a}$. Matrices are written in bold uppercase (e.g. $\mat{M}$). Similarly to vectors, when needed, we will write them with italic font (e.g. $\rdmmat{M}$) to highlight random matrices conditioned on $\vect{x}$. Another step ahead, we note tensors as $\tens{T}$ and in italic font $\rdmmat{T}$ when conditioned on $\vect{x}$. We adopt a uniform notation to represent contractions between tensors and lower-order objects (vectors or matrices), regardless of the contraction direction. Given a tensor $\tens{T} \in \mathbb{R}^{n \times m \times p}$, we denote its contraction with a vector $\vect{a} \in \mathbb{R}^k$ with $ k \in \{n,m,p\}$ as the matrix $(\tens{T} \cdot \vect{a}) := \sum_{i=1}^k T_{\dots i \dots} a_i$ where the index \(i\) corresponds to the dimension being contracted. Similarly, for a matrix \(\mat{A} \in \mathbb{R}^{k \times l}\), where \(k\) and $l$ match distinct dimensions of \(\tens{T}\), we also denote the contraction as: $\tens{T} \cdot \mat{A} := \sum_{i=1,j=1}^{k,l} T_{\dots i \dots j \dots} A_{ij}$. We define the \emph{Frobenius norm} of $\tens{T}$ as $\norm{\tens{T}}_F=\sqrt{\sum_{i=1,j=1,k=1}^{n,m,p} T^2_{ijk}}$. For $\vect{u}_1 \in \mathbb{R}^{n_1}, \dots, \vect{u}_k \in \mathbb{R}^{n_k}$, $(\vect{u}_1 \otimes \dots \otimes \vect{u}_k) \in \mathbb{R}^{n_1 \times \dots \times n_k}$ denotes the usual tensor product $(\vect{u}_1 \otimes \dots \otimes \vect{u}_k)_{i_1 \dots i_k} := u_{1,i_1} \dots u_{k,i_k}$. We denote by $\setim= \{ 1,2 \}$ and $\setit = \{1,3,4\}$. 

Given four estimators $\{\hat{\vect{w}}_{k}\}_{k \in\llbracket 4 \rrbracket}$ for each component $\vect{u}, \vect{v},\vect{x}, \vect{y}$ of the signal,  normalized such that $\| \vect{w}_k \| = \Theta(\sqrt{N_k})$, we define the \emph{mean square error} as $\mathrm{MSE}_k=\norm{\hat{\vect{w}}_k - \vect{w}^\star_k }^2$ where $\para{\vect{w}_1^\star, \vect{w}_2^\star, \vect{w}^\star_3, \vect{w}_4^\star}$ are the planted signals $\para{\vect{u}^\star, \vect{v}^\star, \vect{x}^\star, \vect{y}^\star}$.

\vspace{2mm}
\myparagraph{AMP familly ---}  Formally we define the generic (class of) AMP algorithm(s) for the spiked matrix/tensor models as follows:
\begin{mydef} 
\label{def:AMP}
For  $k \in \llbracket 4 \rrbracket$, let $f_k: \mathbb{R}_+ \times \mathbb{R} \ni (a,b) \mapsto f_k(a,b) \in \mathbb{R}$ Lipschitz and denote by $  \vect{f}(a,\vect{b}) :=(f(a,b_1), \dots, f(a,b_{N_k}))^T \in \mathbb{R}^{N_k}$ its vectorized version with respect to its second argument, $\vect{m}_{0,k} \in \mathbb{R}^{N_k}$ and $\rho>0$, we define the Approximate Message Passing $\mathtt{AMP}( \{\vect{f}_k,\vect{m}_{0,k} \}_{k \in\llbracket4 \rrbracket }, \rho)$ for the spiked matrix-tensor model with denoising functions $\vect{f}_k$ and initializations $\vect{m}_{0,k}$  as the following sequence of iterates for any $t \in \mathbb{N}$:
\begin{align}
\begin{cases}
r^t_k
&=
 \frac{\mathds{1}_{k \in \setim}}{\Delta_m} \alpha_{3-k} \sigma^t_{3-k} 
+  
 \frac{\mathds{1}_{k \in \setit }}{\rho \Delta_t N_1} \;
\sum_{l \in \setit \setminus k}\alpha_l\,  \sigma^t_l \prod_{m \in \setit \setminus \{k,l\} }\langle \vect{w}^{t}_m, \vect{w}^{t-1}_m \rangle ,
\\[5pt]
\vect{b}^t_{k} 
&=
 \frac{\mathds{1}_{k \in \setim }}{\Delta_m \sqrt{N_1}} \; \rdmmat{Y}_m \hat{\vect{w}}^t_{3-k} 
+
 \frac{\mathds{1}_{k \in \setit }}{\rho \Delta_t N_1} \; \rdmtens{Y}_t \cdot \Big(\bigotimes_{l \in \setit \setminus k} \hat{\vect{w}}^t_l \Big)  -  r^t_k \, \hat{\vect{w}}^{t-1}_k \, ,
\\[5pt]
A^t_k 
&=
 \frac{\ind_{k \in \mathcal{I}_m}}{\Delta_m N_1} \| \hat{\vect{w}}^t_{3-k} \|^2 
+
\frac{\ind_{k \in \mathcal{I}_t}}{\rho \Delta_t N_1^2} \prod_{l \in \setit \setminus \{k \}} \| \hat{\vect{w}}^t_{l} \|^2 \, ,
\\[5pt]
\hat{\vect{w}}^{t+1}_k &= \vect{f}_k \para{A^t_k, \vect{b}^t_k} \, , \\[5pt]
\sigma^{t+1}_k &= \frac{1}{N_k} ( \mathrm{div}_{\vect{b}}  \vect{f}_k) \para{A^t_k, \vect{b}^t_k} \, .    
\end{cases}
\end{align}
and with initial conditions $\hat{\vect{w}}^{t=0}_k = \vect{m}_{0,k}$, $\hat{\vect{w}}^{t=-1}_k = \vect{0}$, $\sigma^t_k=0$ for $k \in \llbracket 4 \rrbracket$.   
\end{mydef}
The family of denoising functions $\{ \vect{f}_k \}_{k}$ plays a central role in the AMP iterations by governing the update of the estimator based on the current signal estimate $\vect{b}$ and the noise structure characterized by $A$. Although our results are stated for generic denoising functions, we will primarily focus on two classes, as described below.

\begin{mydef} \label{def:ex_amp} (Denoising functions and associated AMP algorithm)

\begin{itemize}[left=1pt]
    \item \underline{Bayes AMP:} Let $a\geq0$, $\vect{x}, \vect{g} \sim \mathsf{N}(0,\mat{I}_d)$ and independent and $ \vect{x} + \sqrt{a} \vect{g}$ the corresponding Gaussian channel. We call the \emph{Gaussian denoising function} the map $\vect{\eta} : \mathbb{R}_+ \times \mathbb{R}^{d} \ni (a,\vect{b}) \mapsto \vect{\eta}(a,\vect{b}) \in \mathbb{R}^d$ given as the posterior mean for a fixed value of the Gaussian channel:
    \begin{align}
    \label{eq:denoising_bo}
        \vect{\eta}(a,\vect{b}) := \mathbb{E} \{ \vect{x} |  \vect{x} + \sqrt{a} \vect{g} =\vect{b} \} = \frac{1}{1+a} \, \vect{b} \, .
    \end{align}
Furthermore, we call the \emph{Bayes Approximate-Message-Passing} (Bayes-AMP) algorithm, the AMP within the class defined in Def.~\ref{def:AMP}  with $\vect{f}_k=\vect{\eta}$  for any $k \in \llbracket 4 \rrbracket$ and $\rho=1$. 
    \item \underline{Maximum-likelihood AMP:} Let  $\vect{b} \in \R^{N_k}$, the maximum likelihood denoising function is given by
    \begin{align}\label{eq:denoising_ml}
        \vect{\eta}_{\mathrm{ML}} \left(\vect{b}\right) &:= \sqrt{N_k}\frac{\vect{b}}{\norm{\vect{b}}} \; .
    \end{align}
    Setting the denoising function $\vect{f}_k = \vect{\eta}_{\mathrm{ML}}$ for $k \in \llbracket 4 \rrbracket$ within the class of AMP algorithms defined in Def.~\ref{def:AMP} define the \emph{Maximum Likelihood Approximate-Message-Passing} (ML-AMP) algorithm. 
    
    This expression of $\vect{\eta}_{\mathrm{ML}}$ results in a denoising function that is not separable. To address this, we broaden the class of admissible denoisers. This corresponds in our case to replace the divergence $\mathrm{div}_{\vect{b}} \vect{f}_k$ by $\para{\sqrt{N_k} \norm{\vect{b}_k}}^{-1}$. We mention that at leading order these two quantities are equal.

\end{itemize}
\end{mydef}

\textbf{State Evolution} --- \hspace{0.5cm} In the high-dimensional setting, the performance of classical AMP algorithms breaks down to the analysis of a low-dimensional autonomous iteration known as State Evolution (SE), which updates the value of the empirical overlaps over iterations of AMP. 
 For the class of AMP algorithms described in Def.~\ref{def:AMP}, this iteration is given as follows: Let $g_{k} \equiv g_{k, \rho}: \mathbb{R}^4 \ni \vect{x} \mapsto g_{k}(\vect{x}) \in \mathbb{R}$ defined for $k \in \llbracket 4 \rrbracket$ by
    \begin{align}
       g_k\para{\vect{x}} \equiv g_{k, \rho} \para{\vect{x}} &:= \frac{\ind_{k \in \mathcal{I}_m}}{\Delta_m} x_{3-k} + \frac{\ind_{k \in \mathcal{I}_t}}{\rho \Delta_t} \prod_{l \in \setit \setminus \{k \}}x_l\, , \label{eq:g_entry_wise} 
    \end{align} 
then the SE iteration associated to $\mathtt{AMP}(\{ \vect{f}_k, \vect{m}_{0,k} \}_k, \rho)$ is defined for any $t \in \mathbb{N}$ by 
    \begin{align}
    \label{eq:SE}
    \begin{cases}
        q_k^{t+1} &= \mathbb{E}_{ \rdm{V}_k \sim \mathsf{P}_k, \rdm{G}  \sim \mathsf{N}(0,1) } \left[ \eta\para{g_k\para{\vect{q}^t}, g_k\para{\vect{q}^t} \rdm{V}_k + \sqrt{g_k\para{\vect{s}^t}} \rdm{G} } \mathrm{V}_k\right] \, , \\
        s_k^{t+1} &= \mathbb{E}_{ \rdm{V}_k \sim \mathsf{P}_k, \rdm{G}  \sim \mathsf{N}(0,1) }  \left[ \eta\para{g_k\para{\vect{q}^t}, g_k\para{\vect{q}^t} \rdm{V}_k + \sqrt{g_k\para{\vect{s}^t}} \rdm{G} }^2 \right] \, ,
    \end{cases}
    \end{align}    
with initial conditions $\vect{q}^{t=0}=\vect{q}_0$. While the denoising function $\eta$ is not separable in the case of ML-AMP, the correspondence still holds, and the corresponding SE can be found in Eq.~\ref{eq:state_evolution_linearized}. The precise relation between the AMP iterations and its SE is given by the following proposition whose proof is an immediate consequence of \cite{javanmard_state_2013,feng2022unifying}: 
\begin{myprop} Consider the class $\mathtt{AMP}( \{\vect{f}_k,\vect{m}_{0,k} \}_{k \in\llbracket4 \rrbracket }, \rho)$ initialized in such a way that $\langle\vect{m}_{0,k},\vect{w}^\star_k\rangle/N_k \to q_{0,k}$ for $k \in \llbracket 4 \rrbracket$, then  for any $t \in \mathbb{N}$  with $\hat{\vect{w}}^t_k$ as defined in Def.~\ref{def:AMP}, we have $\langle \hat{\vect{w}}_k, \vect{w}^\star_k \rangle/N_k \to q^t_k$ where $\vect{q}^t=(q_1,\dots,q_4)$ follows the SE of Eq.~\eqref{eq:SE} initialized with $\vect{q}^{t=0}=(q_{0,1}, \dots, q_{0,4})$. 
\end{myprop}
Understanding the performance of AMP algorithms—particularly through the phase diagram that borders the parameter regions where each signal component can be inferred—boils down to analyzing the stability of the fixed points of the associated state evolution (SE) equations. Specifically, for the two primary examples described in Def.\ref{def:ex_amp}, these SE equations can be expressed in the following simple form. We recall the definition of the function $\vect{g}(.)=(g_1(.), \dots,g_4(.))$ with $g_k \equiv g_{k,\rho}$ given by Eq.~\eqref{eq:g_entry_wise}, then we have
\begin{itemize}[left=1pt]
    \item \underline{for the Bayes-AMP:} if we set $\vect{f}_{\mathrm{Bayes}}:\mathbb{R}_+^4 \ni \vect{x} \mapsto (f(x_k))_{k \in \llbracket 4 \rrbracket}$ with $f(x)= x/(1+x)$, then the SE of Eq.~\eqref{eq:SE} writes for this case as
    \begin{align}
    \label{eq:state_evolution_BO}
        \vect{q}^{t+1} = \vect{f}_{\mathrm{Bayes}}(\vect{g}_{\rho=1}(\vect{q}^{t})) \, ,
    \end{align}
    \item \underline{for the ML-AMP:} if we set $\vect{f}_{\mathrm{ML}}:\mathbb{R}_+^4 \ni \vect{x} \mapsto (f_k(x_k))_{k \in \llbracket 4 \rrbracket}$ with $f_k(x):= \frac{x}{\sqrt{g_{k,\rho^2}(1)+x^2}}$, then the SE of Eq.~\eqref{eq:SE} writes for this case as:
    \begin{align}
        \vect{q}^{t+1} = \vect{f}_{\mathrm{ML}}(\vect{g}(\vect{q}^{t})) \, .
    \label{eq:state_evolution_linearized}
    \end{align}
\end{itemize}
the fixed points of the state evolution equations are also known to be minimizers of the free energy. Bayes-SE corresponds to the Bayes-optimal free energy derived in App.~\ref{sec:free_e}.

\begin{mydef} \label{def:recovery} Given four estimators $\{\hat{\vect{w}}_{k}\}_{k \in\llbracket 4 \rrbracket}$ for each component $\vect{u}, \vect{v},\vect{x}, \vect{y}$ of the signal,  normalized such that $\| \vect{w}_k \| = \Theta(\sqrt{N_k})$,  we say that we have 
\begin{itemize}
    \item \emph{matrix and tensor weak recovery} if and only if we have both 
    \begin{align}
        \|(\vect{u} \otimes \vect{v})^T(\vect{\hat{w}}_1 \otimes \vect{\hat{w}}_2) \| = \Theta(1) \qquad \mbox{and} \qquad   \|(\vect{x} \otimes \vect{y})^T(\vect{\hat{w}}_3 \otimes \vect{\hat{w}}_4) \| = \Theta(1) \, ,
    \end{align}
    \item  \emph{matrix-only weak recovery} if and only if we have both 
    \begin{align}
        \|(\vect{u} \otimes \vect{v})^T(\vect{\hat{w}}_1 \otimes \vect{\hat{w}}_2) \| = \Theta(1) \qquad \mbox{and} \qquad   \|(\vect{x} \otimes \vect{y})^T(\vect{\hat{w}}_3 \otimes \vect{\hat{w}}_4) \| = o(1) \, ,
    \end{align}
    \item  \emph{no weak recovery} if and only if we have both 
    \begin{align}
        \|(\vect{u} \otimes \vect{v})^T(\vect{\hat{w}}_1 \otimes \vect{\hat{w}}_2) \| = o(1) \qquad \mbox{and} \qquad   \|(\vect{x} \otimes \vect{y})^T(\vect{\hat{w}}_3 \otimes \vect{\hat{w}}_4) \| = o(1) \, .
    \end{align}
\end{itemize}
\end{mydef}

\section{Optimal efficient algorithm: The performance of Bayes AMP}\label{sec:bo_amp}
\emph{Bayes-AMP}, as defined in Def.~\ref{def:ex_amp}, is provably optimal among all first-order algorithms in a large class of inference problems, see \cite{celentano2021high} or conjectures in \cite{gamarnik2022disordered} and its performance in the high dimensional setting is ruled by SE of Eq.~\eqref{eq:state_evolution_BO}. We introduce the $5$-dimensional vector of parameters as
\begin{align}
    \vect{\theta} := (\Delta_m, \Delta_t,\alpha_1,\alpha_2,\alpha_3)^T \in {(\mathbb{R}_+^\star)}^5 \, , 
\end{align}
and partition the parameter region ${(\mathbb{R}_+^\star)}^5$ in three distinct regions:
\begin{align}
    {(\mathbb{R}_+^\star)}^5 = \mathcal{R}_0 \sqcup \mathcal{R}_m \sqcup \mathcal{R}_t \, , 
\end{align}
the \emph{no recovery region} $\mathcal{R}_0$, the \emph{matrix recovery region} $\mathcal{R}_m$ and the \emph{matrix and tensor recovery region} $\mathcal{R}_t$ with 
\begin{align}
    \mathcal{R}_0 \; &:= \; \left\{ \vect{\theta} \in \para{\mathbb{R}_+^\star}^5 \;|\; \Delta_m > \sqrt{\alpha_2}\right\} \, ,  \label{eq:region_0}\\
    \mathcal{R}_m \; &:= \; \left\{ \vect{\theta} \in \para{\mathbb{R}_+^\star}^5 \; | \; \Delta_m < \sqrt{\alpha_2} \; \textrm{and} \; \Delta_t > \delta^{\mathrm{Bayes}}_{c} \right\}\, , \label{eq:region_m} \\
    \mathcal{R}_t \; &:= \; \left\{ \vect{\theta} \in \para{\mathbb{R}_+^\star}^5 \; | \; \Delta_t < \delta^{\mathrm{Bayes}}_{c} \right\}\, ,  \label{eq:region_t}
\end{align}
and where the constant $\delta^{\mathrm{Bayes}}_{c} \equiv \delta^{\mathrm{Bayes}}_{c}(\alpha_2,\alpha_3,\alpha_4,\Delta_m)$ is given by 
\begin{align}
&\label{eq:threshold_BO}\delta^{\mathrm{Bayes}}_{c} 
:=
\;\sqrt{\alpha_3 \alpha_4}\, \frac{\alpha_2 - \Delta_m^2}{\alpha_2+\Delta_m} \, . 
\end{align}
Note that in $\mathcal{R}_t$, the condition $\Delta_m < \sqrt{\alpha_2}$ is implicitly contained by the other constraints. One may notice that the constraint in $\mathcal{R}_0$ is independent of the parameters $(\Delta_t,\alpha_3, \alpha_4)$ of the tensor channel and corresponds to the information-theoretic and algorithmic transition in the matrix channel alone, see \cite{Lesieur_2017}.  An illustration of the phase space partition projected respectively in the $(\Delta_m,\Delta_t)$-plane and the $(\alpha_2,\alpha_3)$-plane is given in Fig.~\ref{fig:no_hard_bo}. Our next result characterizes the fixed point of the associated state evolution.

\begin{mytheo} \label{prop:phase_se_bo} Consider the SE of Eq.~\eqref{eq:state_evolution_BO} then the following holds:
\begin{itemize}[left=15pt]
    \item[(i)] the trivial fixed point $\vect{0}_4=(0,0,0,0)^T$ is a stable fixed point if and only if $\vect{\theta} \in \mathcal{R}_0$;
    \item[(ii)] there exists a unique fixed point of the form $\vect{q}_m=\para{q_{m,1}, q_{m,2}, 0, 0}^T$ with $q_{m,1},q_{m,2} \neq 0$ if and only if $\vect{\theta} \in (\mathbb{R}_+^\star)^5 \setminus \mathcal{R}_0$ and this fixed point is stable if and only if $\vect{\theta} \in \mathcal{R}_m$.  Additionally, the value $q_{m,1}, q_{m,2}$ are given by
        \begin{align} \label{eq:qu_mat_bo}
            q_{m,1} = \frac{\alpha_2 - \Delta_m^2}{\alpha_2+\Delta_m} \; , \hspace{1cm} q_{m,2} = \frac{1}{\alpha_2}\frac{\alpha_2- \Delta_m^2}{1+\Delta_m} \; ;
        \end{align}
    \item[(iii)] there exists at least one stable fixed point $\vect{q}_t=(q_{t,1}, \dots, q_{t,4})^T$  with non-zero entries $q_{t,i} \neq 0$ for $i \in \{1,\dots,4\}$ if $\vect{\theta} \in \mathcal{R}_t$ and furthermore there is no other stable fixed point having a zero component within that region. 
\end{itemize}
\end{mytheo}

\begin{proof}
\leavevmode\vspace{-0.5em}
    \begin{itemize}
        
        \item Existence is immediate and the proof of the stability is given in App.~\ref{subsec:app_null_fp}.
        
        \item The existence, unicity and expression comes from Appendix~\ref{sec:app_sys} with $x_1 = q_u$, $x_2 = q_v$, $a = \alpha_2$, $b = 1$, and $\Delta = \Delta_m$
        
        Additionally, the stability of this fixed point is given by Appendix~\ref{subsec:app_matrix_fp}.
        
        \item If $\vect{\theta} \in \mathcal{R}_t$, then none of $\vect{0}$ nor $\vect{q}_m$ are global minimizer of the free energy because not stable. 
        
        Then, the only possible candidates left are fixed points of the form of $\vect{q}_t$ with $\forall k \; q_{t,k}\neq 0$. Indeed, one can check that any other possibility cannot be fixed point of SE. For instance, the vector $\para{q_1,0,q_2, q_3}$ does not work (meaning that we cannot recover only the tensor). 
        
        Then, $\vect{q}_t$ being the only possible stable fixed points, one of them is the global minimizer of the free energy, in particular, is stable.

        (Indeed, the free energy must have a global minimizer since it is defined on the compact $\left[ 0,1\right]^4$)
    \end{itemize}
\end{proof}

In particular, it yields a transition of order $\mathcal{O}(1)$ for the recovery of both the matrix and the tensor, whereas the recovery of a tensor spike alone occurs only at order $\mathcal{O}(N^{-1/2})$. In fact, the tensor components cannot be recovered independently without recovering the matrix. This result highlights a \emph{staircase} phenomenon \cite{abbe2021staircase}, in which the interplay between matrix and tensor observations shapes the order in which components become learnable.

\begin{figure}[t]
    \centering
    \begin{subfigure}{\textwidth}
        \centering
        \includegraphics[width=0.48\textwidth]{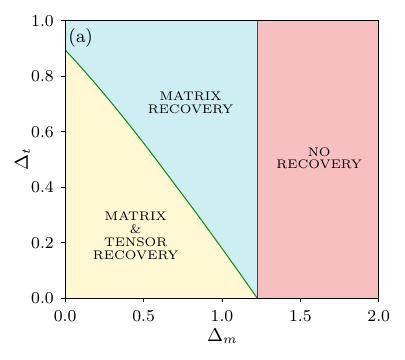}
        \hspace{0.01cm}
        \includegraphics[width=0.48\textwidth]{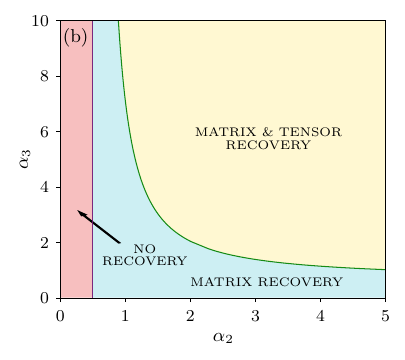}
    \end{subfigure}
    \caption{Phase diagrams for Bayes-AMP algorithm, (a) in $\para{\Delta_m, \Delta_t}$ space  with size ratios $\alpha_2 = 1.5$, $\alpha_3 = 0.8$, and $\alpha_4 = 1$ and in $\para{\alpha_2, \alpha_3}$ space (b) in $\para{\alpha_2, \alpha_3}$ space with noise $\Delta_m = 0.7$, $\Delta_t=0.8$ and a size ratio $\alpha_4=1$. The red, blue and yellow regions respectively delimit $\mathcal{R}_0$, $\mathcal{R}_m$ and $\mathcal{R}_t$. The purple and green solid lines indicate the phase transitions between the regions $\mathcal{R}_0$ and $\mathcal{R}_m$, and between $\mathcal{R}_m$ and $\mathcal{R}_t$, respectively. These transitions are implicitly defined by $\Delta_m = \sqrt{\alpha_2}$ and $\Delta_t = \delta_c^{\mathrm{Bayes}}$.}
    \label{fig:no_hard_bo}
\end{figure}

\section{Maximum Likelihood Estimation}\label{sec:ml_amp}
While it is  good news that there exists an  algorithm that can learn both the matrix and the tensor efficiently, the real  problem for modern machine learning is whether or not it can be done so with gradient descent over the maximum likelihood loss (ML) that reads 
\begin{align}\label{eq:l_tot}
    \mathcal{L}_\rho\para{\vect{u}, \vect{v}, \vect{x}, \vect{y}} = \mathcal{L}_m(\vect{u},\vect{v}) + \frac{1}{\rho}\,\mathcal{L}_t(\vect{u},\vect{x},\vect{y}) \, ,
\end{align}
where
\begin{align}\label{eq:l_m_and_l_t}
    \mathcal{L}_m\para{\vect{u}, \vect{v}} := \frac{1}{2\Delta_m} \norm{\rdmmat{Y}_m - \frac{1}{\sqrt{N_1}} \vect{u}\otimes \vect{v}}_F^2 \, , \qquad \mathcal{L}_{t}\para{\vect{u}, \vect{x}, \vect{y}} := \frac{1}{2\Delta_t} \norm{\rdmtens{Y}_t - \frac{1}{N_1} \vect{u}\otimes\vect{x} \otimes \vect{y}}_F^2  \, .
\end{align}

Here $\rho$ is added to give the possibility for more weights to the tensor or matrix part, but the ML loss, strictly speaking, is for $\rho=1$. Spherical GD with learning rate $\eta$, spherical constraint $\vect{\mu}^t \in \R^4$, initialization $\{\vect{w}^0_k\}_{k=1}^4$  is the  iteration
\begin{align}\label{eq:gd_def}
     \vect{\hat{w}}^{t+1}_k =   \para{1+\mu_k^t}\vect{\hat{w}}^{t}_k - \eta \nabla_{k}  \mathcal{L}_\rho \para{\vect{\hat{w}}^{t}_1 , \vect{\hat{w}}^{t}_2, \vect{\hat{w}}^{t}_3, \vect{\hat{w}}^{t}_4}, \quad  \vect{\hat{w}}^{t=0}_k =  \vect{w}^{0}_k \qquad \mbox{for }  k  \in \llbracket 4 \rrbracket ,
\end{align}

Analyzing the performance of gradient descent algorithms in contexts with rough landscapes—such as those involving spiked tensor models studied in this paper—remains challenging. This analysis requires tools from Dynamical Mean-Field Theory (DMFT), which are often difficult to apply due to the intricate nature of the associated optimization dynamics \cite{sarao_mannelli_who_2019}.   To overcome these analytical challenges, rather than directly analyzing gradient descent, we instead study the family of Maximum Likelihood Approximate Message Passing algorithms defined in Sec.~\ref{sec:gen}. The motivation for this approach—and the introduction of these AMP variants—is that they are specifically designed to share the same fixed points as gradient descent:
\begin{myprop}\label{prop:equiv_GD_AMP_FP}  The fixed points of \emph{ML-AMP} are stationary points of the Gradient Descent with loss $\mathcal{L}$ defined by the iteration of Eq.~ \ref{eq:gd_def}, where each argument is constrained on the sphere of radius $\sqrt{N_k}$.
\end{myprop}

\begin{proof} 

A given fixed point $\hat{\vect{w}}_{0}$ of the ML-AMP algorithm reads
\begin{align}
    \para{\frac{1}{\sqrt{N}}\norm{\vect{b}_k} + r_k} \hat{\vect{w}}_{k,0}  = \frac{\mathds{1}_{k \in \setim }}{\Delta_m \sqrt{N_1}} \; \rdmmat{Y}_m \hat{\vect{w}}_{3-k,0} \; + \; \frac{\mathds{1}_{k \in \setit }}{\rho\Delta_t N_1} \; \rdmtens{Y}_t \cdot \Big(\bigotimes_{l \in \setit \setminus k} \hat{\vect{w}}_{l,0} \Big) \, ,
\end{align}
which is also the stationary case of Eq.\ref{eq:gd_def}. More precisely, the stationary points of GD algorithms can be written as
\begin{align}
    \mu_k \vect{w}_k = \frac{\mathds{1}_{k \in \setim }}{\Delta_m \sqrt{N_1}} \; \rdmmat{Y}_m \vect{w}_{3-k} \; + \;  \frac{\mathds{1}_{k \in \setit }}{\rho \Delta_t N_1} \; \rdmtens{Y}_t \cdot \Big(\bigotimes_{l \in \setit \setminus k} \vect{w}_{l} \Big) \, .
\end{align}
 ML-AMP and GD, by construction, both preserve the spherical constraint at every time iteration.
\end{proof}

We derive state evolution (SE) equations to study the GD fixed points in the \emph{simultaneous} case through the \emph{ML-AMP} setting given by Eq.~\ref{eq:state_evolution_linearized}. Let's introduce the constants
\begin{align}\label{eq:alpha_tilde}
    \tilde{\alpha}_2 :=  \alpha_2 \; \frac{\frac{\alpha_2}{\Delta_m}}{\frac{\alpha_2}{\Delta_m}+\frac{\alpha_3 \alpha_4}{\rho^2\Delta_t}},  \quad \tilde{\delta}_c := \delta^{\mathrm{Bayes}}_c(\tilde{\alpha}_2, \alpha_3,\alpha_4,\Delta_m) = \sqrt{\alpha_3 \alpha_4} \frac{\tilde{\alpha}_2 - \Delta_m^2}{ \tilde{\alpha}_2 + \Delta_m } \, ,
\end{align}
and define 
\begin{align} \label{eq:qu_mat_ml}
\tilde{q}_{m,1} 
= 
\para{\frac{\tilde{\alpha}_2 - \Delta_m^2}{\tilde{\alpha}_2+\Delta_m}}^{1/2} \; ,
\hspace{1cm} 
\tilde{q}_{m,2} 
= 
\para{\frac{1}{\tilde{\alpha}_2}\frac{\tilde{\alpha}_2- \Delta_m^2}{1+\Delta_m}}^{1/2} \; .
\end{align}
We consider the new vector of parameters $\vect{\theta}=(\rho, \alpha_2,\alpha_3, \alpha_4, \Delta_m, \Delta_t)$. Similarly, we introduce the sets no recovery $\tilde{\mathcal{R}}_0 := \{ \vect{\theta} \in (\mathbb{R}_+^\star)^6 | \Delta_m > \sqrt{\tilde{\alpha}_2}\}$, matrix recovery $\tilde{\mathcal{R}}_m := \{ \vect{\theta} \in \para{\mathbb{R}_+^\star}^6  |  \Delta_m < \sqrt{\tilde{\alpha}_2} \,\Delta_t > \tilde{\delta}_{c} \} $, and matrix and tensor recovery $\tilde{\mathcal{R}}_t := \{  \vect{\theta} \in \para{\mathbb{R}_+^\star}^6  | \Delta_t < \tilde{\delta}_c \}$, corresponding to the previous sets $\mathcal{R}_0$, $\mathcal{R}_m$, $\mathcal{R}_t$ with $\alpha_2$ and $\delta_c^{\mathrm{Bayes}}$ replaced by $\tilde{\alpha}_2$ and $\tilde{\delta}_c$ respectively. Our next result characterizes the limit of this simultaneous setting in the following way: 

\begin{mytheo}\label{th:tot_fp_ml}
For the SE equation of Eq.~\ref{eq:state_evolution_linearized} associated to the \emph{ML-AMP} algorithm Theorem \ref{prop:phase_se_bo} holds with $\mathcal{R}_0$, $\mathcal{R}_m$, $\mathcal{R}_t$ replaced respectively by $\tilde{\mathcal{R}}_0$, $\tilde{\mathcal{R}}_m$, $\tilde{\mathcal{R}}_t$ and $q_{m,1},q_{m,2}$ replaced respectively by $\tilde{q}_{m,1}$ and $\tilde{q}_{m,2}$.
\end{mytheo}

\begin{proof} 
    The proof of this result is deferred to App.~\ref{sec:app_stab}.
\end{proof}
At first glance, the observation that all thresholds remain of order $\mathcal{O}(1)$ highlights the \emph{staircase} phenomenon, and the possible recovery of the tensor, despite the non-Bayesian approach. However, a closer inspection is revealing a less than ideal situation: In fact, the ERM is yielding always a worse phase diagram, and attempting to learn the tensor makes it harder to learn the matrix! 

From the expression of Eq.~\ref{eq:alpha_tilde} for $\tilde{\alpha}_2$, it is clear that $\tilde{\alpha}_2 \leq \alpha_2$, and similarly, one finds that $\tilde{\delta}_c \leq \delta^{\mathrm{Bayes}}_c$. As a consequence—by definition of the corresponding regions—it follows that $\tilde{\mathcal{R}}_0 \supseteq \mathcal{R}_0$ and $\tilde{\mathcal{R}}_t \subseteq \mathcal{R}_t$. In other words, the detection threshold for the matrix component (and for the joint matrix-tensor component) is always higher—i.e., detection is harder—than in the Bayesian setting discussed in the previous section. This is illustrated in Fig.~\ref{fig:no_hard_ml}, which shows the phase diagram with the different regions projected onto the $(\Delta_m, \Delta_t)$ and $(\alpha_2, \alpha_3)$ planes.

This effect stems from the fact that the matrix and tensor losses are respectively scaled by $1/\Delta_m$ and $1/\Delta_t$ within the total loss $\mathcal{L}$. As $\Delta_t$ decreases, the tensor loss $\mathcal{L}_t$ gains more weight, effectively roughening the loss landscape and hindering optimization. To mitigate this effect, one can try to adjust the parameter $\rho$. Since $\tilde{\alpha}_2$ is a decreasing function of $\rho$, increasing $\rho$ reduces the relative weight of the tensor loss in the total loss, which improves the matrix recovery threshold, but it remains always worse than the Bayesian one for any values of $\rho$. 

This problem thus gives an example where the ML weak recovery is  worse than the Bayesian one. Here the joint optimization over both the tensor and the matrix  significantly impairs recovery performance, as the attempt at the fitting of the tensor degrades the matrix reconstruction quality. Despite the tensor inherently providing richer information, simultaneous learning introduces complex interdependencies that shift the weak recovery threshold, rendering matrix estimation—a typically straightforward task—substantially more challenging. This was avoided by the Bayes AMP algorithm, which avoids overfitting and proceeds cautiously due to its Bayesian nature. 

\begin{figure}[t]
    \centering
    \begin{subfigure}{\textwidth}
        \centering
        \includegraphics[width=0.48\textwidth]{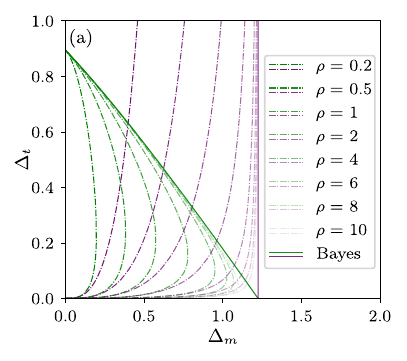}
        \hspace{0.01cm}
        \includegraphics[width=0.48\textwidth]{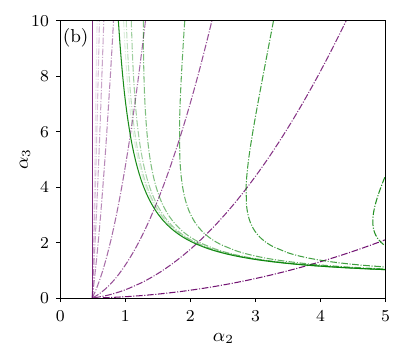}
    \end{subfigure}
    \caption{Phase diagrams for ML-AMP algorithm in $\para{\Delta_m, \Delta_t}$ space (a) with size ratios $\alpha_2 = 1.5$, $\alpha_3 = 0.8$, and $\alpha_4 = 1$ and in $\para{\alpha_2, \alpha_3}$ space (b) with noise $\Delta_m = 0.7$, $\Delta_t=0.8$ and a size ratio $\alpha_4=1$. The purple and green dashed-dotted lines indicate the phase transitions between the regions $\tilde{\mathcal{R}}_0$ and $\tilde{\mathcal{R}}_m$, and between $\tilde{\mathcal{R}}_m$ and $\tilde{\mathcal{R}}_t$, respectively, for several values of the parameter $\rho$. These transitions are implicitly defined by $\Delta_m = \sqrt{\tilde{\alpha}_2}$ and $\Delta_t = \tilde{\delta}_c$. Note that for any value of $\rho$ the transition are worst with respect to the Bayesian algorithm.}
    \label{fig:no_hard_ml}
\end{figure}

\section{Optimal transitions with matrix and contracted tensor PCA}\label{sec:spec}

To address this limitation, one could try to preserve the matrix recovery through a \emph{sequential GD optimization} where one first recovers the matrix alone, preserving the original recovery BBP threshold. Then, the resulting estimator will be used to help the tensor optimization. Actually, this corresponds to the limit $\rho \to \infty$ in the previous approach, which indeed recovers the Bayesian optimal results.

Interestingly, this can be implemented directly, without gradient descent,  
by two successive spectral algorithms where one performs PCA on the matrix, plugs the estimate of ${\bf u}$ in the tensor, and applies PCA on the tensor to estimate ${\bf x}$ and ${\bf y}$. While this could be directly studied by AMP through its equivalence to gradient descent Proposition \ref{prop:equiv_GD_AMP_FP}, it can be directly studied by random matrix theory as a PCA phenomenon. This leads to the following result: first, the optimal Bayesian thresholds are recovered, and thus the sequential approach yields optimal thresholds. Secondly, one can characterize the performance. Define 
\begin{align}
    q^{\mathrm{s}}_{1}
    = 
    \para{\frac{\alpha_2 - \Delta_m^2}{\alpha_2+\Delta_m}}^{1/2} \; , &\hspace{1cm} 
    q^{\mathrm{s}}_{2} 
    = 
    \para{\frac{1}{\alpha_2}\frac{\alpha_2- \Delta_m^2}{1+\Delta_m}}^{1/2} \,  , 
    \label{eq:overlaps_12_seq} \\
    q^{\mathrm{s}}_{3} 
    = 
    \para{\frac{1}{\alpha_3 (q^{\mathrm{s}}_{1})^2} \frac{\alpha_3 \alpha_4 (q^{\mathrm{s}}_{1})^4 - \Delta_t^2}{\alpha_4 (q^{\mathrm{s}}_{1})^2 + \Delta_t}}^{1/2} \;
     ,&  \hspace{1cm}
    q^{\mathrm{s}}_4
    = 
    \para{\frac{1}{\alpha_4 (q^{\mathrm{s}}_{1})^2} \frac{\alpha_3 \alpha_4 (q^{\mathrm{s}}_{1})^4 - \Delta_t^2}{\alpha_3 (q^{\mathrm{s}}_{1})^2 + \Delta_t}}^{1/2} \, . \label{eq:overlaps_34_seq}
\end{align}
Then we have:

\begin{mytheo} \label{th:svd} Let $\rdmvect{\hat{u}}^{SVD}$, $\rdmvect{\hat{v}}^{SVD}$ the top (unit-norm) singular vectors of $\rdmmat{Y}_m/\sqrt{N_2}$ ; and $\rdmvect{\hat{x}}^{SVD}$, $\rdmvect{\hat{y}}^{SVD}$ the top (unit-norm) singular vectors of the matrix $( \rdmtens{Y_t} \cdot \rdmvect{\hat{u}}^{SVD})/\sqrt{N_4}$, then  we have
\begin{align}
    (i) \qquad  |\langle \sqrt{N_1} \, \rdmvect{\hat{u}}^{SVD} , \vect{u}^\star \rangle|/N_1 \to q^{\mathrm{s}}_1 , \qquad 
    |\langle \sqrt{N_2} \, \rdmvect{\hat{v}}^{SVD} , \vect{v}^\star \rangle|/N_2 \to q^{\mathrm{s}}_2 , \\
    (ii) \qquad |\langle \sqrt{N_3} \, \rdmvect{\hat{x}}^{SVD} , \vect{x}^\star \rangle|/N_3 \to q^{\mathrm{s}}_3 , \qquad
    |\langle \sqrt{N_4} \, \rdmvect{\hat{y}}^{SVD} , \vect{y}^\star \rangle|/N_4 \to q^{\mathrm{s}}_4, 
\end{align}
where $q^{\mathrm{s}}_1, q^{\mathrm{s}}_2,q^{\mathrm{s}}_3, q^{\mathrm{s}}_4$ are defined in Eq.~\ref{eq:overlaps_12_seq}~and~\ref{eq:overlaps_34_seq}. 
\end{mytheo}

\begin{proof}
    We recall the following result regarding spiked matrix:

\begin{adjustwidth}{2em}{0pt}
    {\begin{mylemma}[BBP phase-transition] \label{lemma:bbp}Let $\rdmmat{X}$ a $(n \times m)$ matrix with independent elements, with mean zero and variance $\Delta$ and let $\mat{P} = \sigma_n \vect{u}\vect{v}^T$ a rank-one matrix with $\vect{u},\vect{v}$ of unit norm and independent of  $\rdmmat{X}$, and the non-trivial singular value satisfying $\sigma_n \to \sigma$ for some constant $\sigma$, then as $n \to \infty$ with $n/d \to \alpha$, we have  
    \begin{align}
    \Big|\Big\langle \rdmvect{u}_1 \Big( \rdmmat{X}/\sqrt{m} + \mat{P} \Big) , \vect{u}  \Big \rangle \Big|^2 
    &\xrightarrow[n \to \infty]{\mathrm{a.s.}} 
    1 - \frac{\alpha (1 + \sigma^2/\Delta)}{\sigma^2/\Delta(\sigma^2/\Delta + \alpha)}
    \\
    \Big|\Big\langle \rdmvect{v}_1 \Big( \rdmmat{X}/\sqrt{m} + \mat{P} \Big) , \vect{v}  \Big \rangle \Big|^2
    &\xrightarrow[n \to \infty]{\mathrm{a.s.}} 
    1 - \frac{(\alpha + \sigma^2/\Delta)}{\sigma^2/\Delta(\sigma^2/\Delta + 1)}
    \end{align}
    where $\rdmvect{u}_1$ and $\rdmvect{v}_1$ denote the top singular vectors of the corresponding matrix, normalized to have norm one. 
    \end{mylemma} 
    }
\end{adjustwidth}
Thus we have:

(i) For the spiked (asymmetric) rectangular matrices given by Eq.~\ref{eq:uts_setting_m}, applying Lemma \ref{lemma:bbp} with $\Delta = \Delta_m$, $\vect{u} = \vect{u}^\star/\|\vect{u}^\star\|$, $\vect{v} = \vect{v}^\star/\|\vect{v}^\star\|$ yields the desired result of the theorem for $k=1,2$. 

(ii) Next, we decompose for the contracted tensor as
\begin{align}
    (\rdmtens{Y_t} \cdot \rdmvect{\hat{u}})/\sqrt{N_4} = \rdmmat{\Tilde{X}}/\sqrt{N_4} + \mat{\Tilde{P}} \, ,
\end{align}
where we use the notations $\rdmvect{\hat{u}} \equiv \rdmvect{u}_1 \para{ \rdmmat{Y}_m / \sqrt{N_2} }$ and where the noise matrix $\rdmmat{\Tilde{X}}$ has elements $\sqrt{\Delta_t} (\rdmtens{Z}_t \cdot \rdmvect{\hat{u}})_{jk} = \sqrt{\Delta_t} \sum_{a} \hat{u}_a Z_{ajk}$ which are independent (conditioned on $\rdmvect{\hat{u}}$) with mean zero and variance $\Delta_t$ and  the signal part is given by the rank-one matrix $\mat{\Tilde{P}} =  \tilde{\sigma}_n\; 
\vect{x}^\star (\vect{y}^\star)^T / \para{\|\vect{x}^\star \| \|\vect{y}^\star \|}
$ with $\tilde{\sigma}_n := \langle \rdmvect{\hat{u}}, \vect{u}^\star/\sqrt{N_1} \rangle \; \| \vect{x}^\star \| \| \vect{y}^\star \| / \sqrt{N_1 N_4}  $ which by the law of large numbers and the previous part $(i)$ converges to $\tilde{\sigma}_n \to q_1^2 \sqrt{\alpha_3 \alpha_4}$, giving the remaining result for $k=3,4$ by Lem.~\ref{lemma:bbp}.
\end{proof}

We illustrate these asymptotic results with finite size predictions in the Fig.~\ref{fig:spectral}. 
This highlights the advantage of sequential learning over joint optimization. A simple spectral method—applying PCA first on the matrix, followed by the tensor—effectively decouples the estimation, avoiding the pitfalls observed in simultaneous learning.

\begin{figure}[t]
    \centering
    \includegraphics[width=0.8\textwidth]{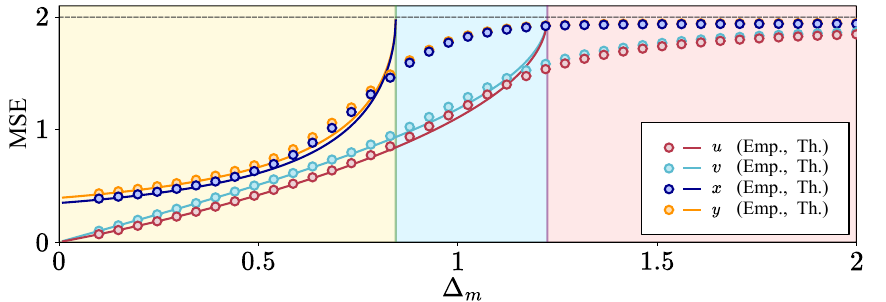}
    \caption{Empirical and theoretical MSE vs. $\Delta_m$ with parameters $\Delta_t=0.3,\; N_1=10^3, \; N_2 = 1.5\,.\,10^3,\; N_3 = 0.8\,.\,10^3,\;N_4 = 10^3$. The empirical values (dots) are computed numerically via the the sequential spectral method, averaged over 500 samples. The error bars are not visible on the plots due to their small magnitude. Their theoretical predictions (solid lines) are given by the overlaps $(q^{\mathrm{s}}_1, q^{\mathrm{s}}_2, q^{\mathrm{s}}_3, q^{\mathrm{s}}_4)$ defined in Eq.~\ref{eq:overlaps_12_seq}~and~\ref{eq:overlaps_34_seq}. The background colors match the three regions $\mathcal{R}_0$, $\mathcal{R}_m$, and $\mathcal{R}_t$. The purple and green vertical lines indicate the phase transitions between the regions $\mathcal{R}_0$ and $\mathcal{R}_m$, and between $\mathcal{R}_m$ and $\mathcal{R}_t$, respectively. These transitions are implicitly defined by $\Delta_m = \sqrt{\alpha_2}$ and $\Delta_t = \delta_c^{\mathrm{Bayes}}$.
    }
    \label{fig:spectral}
\end{figure}

\section{Conclusion}

We studied a non-symmetric matrix-tensor model with a shared signal using several complementary tools. We first derived a Bayesian Approximate Message Passing algorithm and analyzed its performance through its associated state evolution equations. This analysis revealed three distinct recovery regions: no recovery, matrix-only recovery, and joint matrix-tensor recovery. We computed explicit thresholds separating these regimes and identified a staircase phenomenon, where the matrix acts as a stepping stone for tensor recovery.

We then turned to a gradient descent approach on the maximum likelihood loss, using AMP as a proxy to analyze its performance. The same phase structure appears, but recovery becomes more difficult: when the tensor dominates the loss, its higher complexity degrades the recovery of all components. This effect can be reduced by decreasing its influence, but the resulting performance remains strictly weaker than that of the Bayes-initialized AMP.

These results motivate a simple sequential strategy: first recover the matrix, then use the estimate to infer the tensor. This approach reduces to two spectral methods and allows one to recover the Bayesian thresholds, despite lower overlaps. It highlights the importance of exploiting the structure of the model and using a form of curriculum in the learning process—estimating easier components first to guide the rest.

\section*{Acknowledgment} 
We acknowledge funding from the Swiss National Science Foundation grants SNFS  SMArtNet (grant number 212049), OperaGOST (grant number 200021 200390) and DSGIANGO (grant number 225837).

\bibliography{bib_converted_from_cleaned}

\newpage 

\appendix

\section{Detailed expressions of the algorithms}\label{sec:app_algo}

In this appendix, we provide the explicit iterative forms of the three AMP algorithms analyzed in the main text: Bayes-AMP, ML-AMP, and Sequential ML-AMP. While Definition~\ref{def:AMP} presents a unified formulation that covers all cases, the goal here is to write each instance in detail. The notations used are consistent with those introduced earlier, except that we replace the generic estimator $\hat{\vect{w}}_k$ with signal-specific symbols—for example, $\hat{\vect{w}}_1$ is now written as $\hat{\vect{u}}$, and similarly for the other components. This change is purely notational and aims to enhance readability when tracking updates for each variable separately.

\subsection{Bayes Approximate Message Passing}\label{subsec:app_bayes_AMP}

In a Bayesian setting, the algorithm has full access to the generative model, including the prior distributions over the latent variables. This setting allows for principled inference by computing the posterior distribution. Thus, a Bayesian algorithm aims at maximizing the posterior of the model which takes the following form:
\begin{align}\label{eq:post}
    P(\vect{u}, \vect{v}, \vect{x}, \vect{y} | \rdmmat{Y}_m, \rdmtens{Y}_t) = \frac{e^{-\mathcal{L}_{\rho=1}(\vect{u}, \vect{v}, \vect{x}, \vect{y})-\frac{1}{2}(\norm{\vect{u}}^2 +\norm{\vect{v}}^2+\norm{\vect{x}}^2 +\norm{\vect{y}}^2)}}{\mathcal{Z}(\rdmmat{Y}_m, \rdmtens{Y}_t)}\, ,
\end{align}
where $\mathcal{L}_{\rho=1} = \frac{1}{2\Delta_m} \norm{\rdmmat{Y}_m - \frac{1}{\sqrt{N_1}} \vect{u}\otimes \vect{v}}_F^2 + \frac{1}{2\Delta_t} \norm{\rdmtens{Y}_t - \frac{1}{N_1} \vect{u}\otimes\vect{x} \otimes \vect{y}}_F^2$ is the negative log likelihood defined in the main text. The denominator $\mathcal{Z}$ is the normalization constant, or partition function. 

To make the structure of dependencies explicit, we represent this posterior as a factor graph. Each node corresponds to a latent variable, and each factor encodes either a prior or an observation-induced interaction. The figure below shows a simplified version. Blank squares indicate that additional couplings exist (notably due to the fully connected nature of the matrix and tensor observations), but are omitted for the sake of clarity:
\begin{figure}[ht]
\centering
\begin{tikzpicture}[scale=0.8, transform shape, every node/.style={font=\normalsize}]

    \centering
    \tikzset{var/.style={ellipse, draw, minimum width=2.5cm, minimum height=1.5cm}}
    \tikzset{factor_prior/.style={rectangle, draw, minimum size=0.5cm}}-
    \tikzset{factor_observe/.style={rectangle, draw, minimum size=1cm}}-
    \tikzset{factor_blank/.style={rectangle, draw, minimum size=0.5cm}}-
    \tikzset{main/.style={circle, draw, minimum size=1cm}}

    \node[main] (v1) at (-6,0) {$v_{1}$};
    \node[main] (u1) at (-2,0) {$u_{1}$};
    \node[main] (x1) at (2,0) {$x_{1}$};
    \node[main] (y1) at (6,0) {$y_{1}$};

    \node[main] (vj) at (-6,-3) {$v_{j}$};
    \node[main] (ui) at (-2,-3) {$u_{i}$};
    \node[main] (xj) at (2,-3) {$x_{j}$};
    \node[main] (yk) at (6,-3) {$y_{k}$};

    \node[main] (vn) at (-6,-6) {$v_{N_2}$};
    \node[main] (un) at (-2,-6) {$u_{N_1}$};
    \node[main] (xn) at (2,-6) {$x_{N_3}$};
    \node[main] (yn) at (6,-6) {$y_{N_4}$};

    \node at (-6, -1) {$\vdots$};
    \node at (-2, -1) {$\vdots$};
    \node at (6, -1) {$\vdots$};

    \node at (-6, -4) {$\vdots$};
    \node at (-2, -4) {$\vdots$};
    \node at (6, -4) {$\vdots$};

    \node[factor_observe] (Yij) at (-4, -4.1) {$Y_{m,\,i j}$};
    \node[factor_blank] (Y1j) at (-4, -1.1) {};
    \node[factor_observe] (Tijk) at (2, -4.1) {$Y_{t,\, i j k}$};
    \node[factor_blank] (Tn1k) at (2, -1.1) {};

    \node[factor_prior] (Pv1) at (-6,1) {$P_v$};
    \node[factor_prior] (Pu1) at (-2,1) {$P_u$};
    \node[factor_prior] (Px1) at (2,1) {$P_x$};
    \node[factor_prior] (Py1) at (6,1) {$P_y$};
    \node[factor_prior] (Pvj) at (-6,-2) {$P_v$};
    \node[factor_prior] (Pui) at (-2,-2) {$P_u$};
    \node[factor_prior] (Pxj) at (2,-2) {$P_x$};
    \node[factor_prior] (Pyk) at (6,-2) {$P_y$};
    \node[factor_prior] (Pvn) at (-6,-5) {$P_v$};
    \node[factor_prior] (Pun) at (-2,-5) {$P_u$};
    \node[factor_prior] (Pxn) at (2,-5) {$P_x$};
    \node[factor_prior] (Pyn) at (6,-5) {$P_y$};

    \draw (ui) -- (Yij);
    \draw (vj) -- (Yij);
    \draw (u1) -- (Y1j);
    \draw (vj) -- (Y1j);
    \draw (ui) -- (Tijk);
    \draw (xj) -- (Tijk);
    \draw (yk) -- (Tijk);
    \draw (un) -- (Tn1k);
    \draw (x1) -- (Tn1k);
    \draw (yk) -- (Tn1k);

    \draw (v1) -- (Pv1);
    \draw (u1) -- (Pu1);
    \draw (x1) -- (Px1);
    \draw (y1) -- (Py1);
    \draw (vj) -- (Pvj);
    \draw (ui) -- (Pui);
    \draw (xj) -- (Pxj);
    \draw (yk) -- (Pyk);
    \draw (vn) -- (Pvn);
    \draw (un) -- (Pun);
    \draw (xn) -- (Pxn);
    \draw (yn) -- (Pyn);
\end{tikzpicture}
\caption{Factor graph of the posterior distribution. The displayed couplings are not exhaustive; blank factors illustrate other interactions not shown here, beyond the main ones linking $u_i$, $v_j$, $x_j$, and $y_k$.}
\label{fig:factor_graph}
\end{figure}
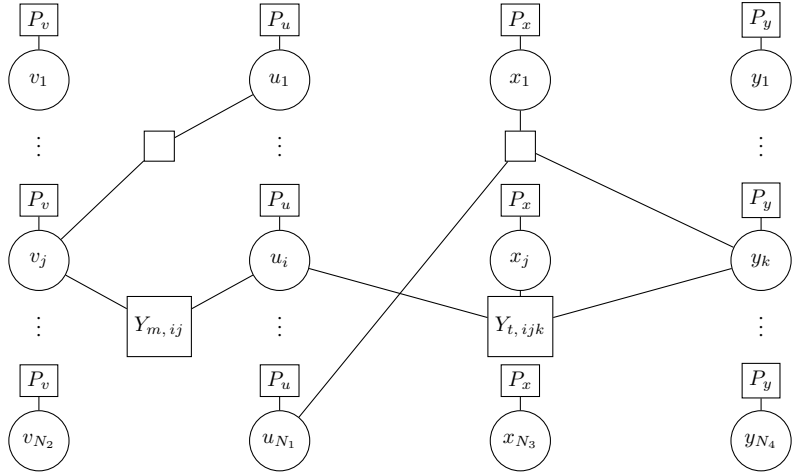

From the factor graph, one can formally write the Belief Propagation (BP) equations, which iteratively approximate the marginals of the posterior distribution. For clarity, we denote the four latent variables as $w_1 = u$, $w_2 = v$, $w_3 = x$, and $w_4 = y$. We introduce the messages related to the estimator $w_k$: $m_{k, i \to ij}$ and $\tilde{m}_{k, ij \to i}$ for the matrix part (when available), and $t_{k, i \to ijk}$ and $\tilde{t}_{k, ijk \to i}$ for the tensor part (when available), which encode the information exchanged between variable nodes and observation factors.

We now give the BP update rules for the variable $u$; the updates for $v$, $x$, and $y$ follow the same pattern.
\begin{align}
    &m_{1, i \to i j } (u_i) = \frac{1}{Z_{1, i \to i j }} \prod_{ \substack{1 \leq j' \leq N_2 \\ j' \neq j}} \tilde{m}_{1, ij' \to i} (u_i) \prod_{\substack{ 1\leq j' \leq N_3 \\ 1\leq k' \leq N_4}} \tilde{t}_{1, i j' k' \to i} (u_i) \label{eq:BP_1} \\[5pt]
    &\tilde{m}_{1, i j  \to i} = \frac{1}{Z_{1, i j  \to i}} \int dv_j  \; m_{2, j \to i j }(v_j) e^{-\frac{1}{2\Delta_m} \norm{\rdmmat{Y}_m - \frac{1}{\sqrt{N_1}} \vect{u}\otimes \vect{v}}_F^2} \label{eq:BP_2} \\[5pt]
    &t_{1, i \to i j k} (u_i) = \frac{1}{Z_{1, i \to i j k}} \prod_{1 \leq  j' \leq N_2 } \tilde{m}_{1,i j' \to i} (u_i) \prod_{\substack{ 1\leq j' \leq N_3 \\ 1\leq k' \leq N_4 \\ j' \neq j ,\; k' \neq k}} \tilde{t}_{1, i j' k' \to i} (u_i) \label{eq:BP_3} \\[5pt]
    &\tilde{t}_{1, i j k \to i } = \frac{1}{Z_{1, i j k \to i}} \int dx_j dy_k  \; t_{3, j \to i j k }(x_j) t_{4, k \to i j k }(y_k)  e^{-\frac{1}{2\Delta_t} \norm{\rdmtens{Y}_t - \frac{1}{N_1} \vect{u}\otimes\vect{x} \otimes \vect{y}}_F^2} \label{eq:BP_4}
\end{align}

BP equations of this form are generally intractable to compute exactly. However, under suitable approximations, these equations can be simplified into the relaxed BP framework. From there, Approximate Message Passing (AMP) can be derived by computing Onsager reaction terms and leveraging statistical properties such as self-averaging and Bayesian setting. The derivation closely follows the procedure established for the spiked matrix model in \cite{Lesieur_2017}, although the computations are heavier in our matrix-tensor setting. Still, the recipe remains the same.

We now present the final form of the Bayes-AMP updates:
\begin{align}
\hspace{0pt}
\begin{array}{rl}
\vcenter{\hbox{for $\hat{\vect{u}}$}} &
\left\{
\begin{aligned}
r_1^t &= \frac{\alpha_2}{\Delta_m} \sigma^t_2 + \frac{\alpha_3 \alpha_4}{\Delta_t} \, \para{ \frac{1}{N_3} \langle  \hat{\vect{x}}^t, \hat{\vect{x}}^{t-1} \rangle\sigma^t_4 + \frac{1}{N_4}\langle  \hat{\vect{y}}^t, \hat{\vect{y}}^{t-1} \rangle \sigma^t_3} \, , 
\\[5pt]
\vect{b}^t_1 &= \frac{1}{\Delta_m\sqrt{N_1}} \;\rdmmat{Y}_m \; \hat{\vect{v}}^t  \;
+ \frac{1}{\Delta_t N_1}\;\rdmtens{Y}_{t} \,\cdot\, \para{\hat{\vect{x}}^t \otimes \hat{\vect{y}}^t} \, - \, r_1^t \hat{\vect{u}}^{t-1} \, ,
\\[5pt]
A^t_1 &= \frac{1}{\Delta_m N_1} \norm{\hat{\vect{v}}^t}^2 + \frac{1}{\Delta_t N_1^2} \norm{\hat{\vect{x}}^t}^2\, \norm{\hat{\vect{y}}^t}^2  \, , 
\\[5pt]
\hat{\vect{u}}^{t+1} &= \frac{\vect{b}^t_1}{1+A_1^t} \, , 
\\[5pt]
\sigma_1^{t+1} &= \frac{1}{1+A_1^t} \, .
\end{aligned}
\right.
\end{array}
\end{align}
\begin{align}
\hspace{-115pt} 
\begin{array}{rl}
\vcenter{\hbox{for $\hat{\vect{v}}$}} &
\left\{
\begin{aligned}
r_2^t &= \frac{1}{\Delta_m} \sigma^t_1 \, , 
\\[5pt]
\vect{b}^t_2 &= \frac{1}{\Delta_m\sqrt{N_1}} \;\rdmmat{Y}_m \; \hat{\vect{u}}^t \, - \, r_2^t\, \hat{\vect{v}}^{t-1} \, ,
\\[5pt]
A^t_2 &= \frac{1}{\Delta_m N_1} \norm{\hat{\vect{u}}^t}^2  \, , 
\\[5pt]
\hat{\vect{v}}^{t+1} &= \frac{\vect{b}^t_2}{1+A_2^t} \, , 
\\[5pt]
\sigma_2^{t+1} &= \frac{1}{1+A_2^t} \, .
\end{aligned}
\right.
\end{array}
\end{align}
\begin{align}
\hspace{-50pt} 
\begin{array}{rl}
\vcenter{\hbox{for $\hat{\vect{x}}$}} &
\left\{
\begin{aligned}
r_3^t &= \frac{\alpha_4}{\Delta_t} \, \para{ \frac{1}{N_1} \langle  \hat{\vect{u}}^t, \hat{\vect{u}}^{t-1} \rangle\sigma^t_4 + \frac{1}{N_4}\langle  \hat{\vect{y}}^t, \hat{\vect{y}}^{t-1} \rangle \sigma^t_1} \, , 
\\[5pt]
\vect{b}^t_3 &= \frac{1}{\Delta_t N_1}\;\rdmtens{Y}_{t} \,\cdot\, \para{\hat{\vect{u}}^t \otimes \hat{\vect{y}}^t} \, - \, r_3^t\, \hat{\vect{x}}^{t-1} \, ,
\\[5pt]
A^t_3 &= \frac{1}{\Delta_t N_1^2} \norm{\hat{\vect{u}}^t}^2 \, \norm{\hat{\vect{y}}^t}^2  \, ,
\\[5pt]
\hat{\vect{x}}^{t+1} &= \frac{\vect{b}^t_3}{1+A_3^t} \, ,
\\[5pt]
\sigma_3^{t+1} &= \frac{1}{1+A_3^t} \, .
\end{aligned}
\right.
\end{array}
\end{align}
\begin{align}
\hspace{-50pt} 
\begin{array}{rl}
\vcenter{\hbox{for $\hat{\vect{y}}$}} &
\left\{
\begin{aligned}
r_4^t &= \frac{\alpha_3}{\Delta_t} \, \para{ \frac{1}{N_1} \langle  \hat{\vect{u}}^t, \hat{\vect{u}}^{t-1} \rangle\sigma^t_3 + \frac{1}{N_3}\langle  \hat{\vect{x}}^t, \hat{\vect{x}}^{t-1} \rangle \sigma^t_1} \, , 
\\[5pt]
\vect{b}^t_4 &= \frac{1}{\Delta_t N_1}\;\rdmtens{Y}_{t} \,\cdot\, \para{\hat{\vect{u}}^t \otimes \hat{\vect{x}}^t} \, - \, r_4^t\, \hat{\vect{y}}^{t-1} \, ,
\\[5pt]
A^t_4 &= \frac{1}{\Delta_t N_1^2} \norm{\hat{\vect{u}}^t}^2 \, \norm{\hat{\vect{x}}^t}^2  \, ,
\\[5pt]
\hat{\vect{y}}^{t+1} &= \frac{\vect{b}^t_4}{1+A_4^t} \, ,
\\[5pt]
\sigma_4^{t+1} &= \frac{1}{1+A_4^t} \, .
\end{aligned}
\right.
\end{array}
\end{align}

From this algorithm, one can derive a low-dimensional time-recursive equation to predict the performance of Bayes-AMP. Namely, the Bayes state evolution (Bayes-SE) can be rigorously derived following \cite{javanmard_state_2013}, and leads to the generic form stated in Eq.~\ref{eq:SE}. Nevertheless, we explicitly derive the Bayes-SE in this appendix in order to obtain the specific form given in Eq.~\ref{eq:state_evolution_BO}. 

In general, one can write:
\begin{align}
    q_1^{t+1} = \frac{1}{N_1}\langle \hat{\vect{u}}^{t+1}, \vect{u}^\star \rangle = \frac{1}{N_1}\langle \vect{\eta}\para{A_1^t, \vect{b}_1^t}, \vect{u}^\star \rangle \, .
\end{align}
However, in the Bayesian setting one can remember that $q^t_k=\norm{\hat{\vect{w}}^t_k}^2/N_k$ (or equivalently, $\vect{q}^t=\vect{s}^t$ with the variables introduced in Eq.~\ref{eq:SE}), it follows:
\begin{align}
\begin{cases}
    b^t_{1,i} = \mathbb{E}\left[ b^t_{1,i} \right] + \sqrt{\operatorname{Var}\para{b^t_{1,i}}}\,G_i = [\vect{g}(\vect{q^t})]_1 u_i^\star + \sqrt{[\vect{g}(\vect{q^t})]_1} \, G_i \, , \\[5pt]
    A^t_1 = \mathbb{E}\left[ A^t_1 \right] + \sqrt{\operatorname{Var}\para{A_1^t}}\, G_i = [\vect{g}(\vect{q^t})]_1 \, , 
\end{cases}
\end{align}
with the gaussian random vector $G_i \overset{\text{i.i.d.}}{\sim} \mathcal{N}(0,1)$, and we recall $[\vect{g}(\vect{q^t})]_1 = \frac{\alpha_2}{\Delta_m}q_2^t + \frac{\alpha_3\alpha_4}{\Delta_t}q_3^tq_4^t$. Thus, one can write the closed-form update:
\begin{align}
    q_1^{t+1} = \mathbb{E}_{u_i^\star, G_i}\left[ \eta\para{[\vect{g}(\vect{q^t})]_1, \;[\vect{g}(\vect{q^t})]_1 u_i^\star + \sqrt{[\vect{g}(\vect{q^t})]_1} \, G_i} u_i^\star\right] = \frac{[\vect{g}(\vect{q^t})]_1}{1+[\vect{g}(\vect{q^t})]_1} \, .
\end{align}
Repeating this procedure for the remaining variables yields the full Bayes-SE system:
\begin{align}\label{eq:app_bo_se}
\hspace{-117pt} 
    \begin{cases}
        q_1^{t+1} = \frac{\frac{\alpha_2}{\Delta_m}q_2^t +  \frac{\alpha_3\alpha_4}{\Delta_t}q_3^t q_4^t}{1 + \frac{\alpha_2}{\Delta_m}q_2^t +  \frac{\alpha_3\alpha_4}{\Delta_t}q_3^t q_4^t} \, ,
        \\[10pt]
        q_2^{t+1} = \frac{\frac{1}{\Delta_m}q_1^t}{1 + \frac{1}{\Delta_m}q_2^t} \, ,
        \\[10pt]
        q_3^{t+1} = \frac{\frac{\alpha_4}{\Delta_t}q_1^t q_4^t}{1 + \frac{\alpha_4}{\Delta_t}q_1^t q_4^t} \, ,
        \\[10pt]
        q_4^{t+1} = \frac{\frac{\alpha_3}{\Delta_t}q_1^t q_3^t}{1 + \frac{\alpha_3}{\Delta_t}q_1^t q_3^t} \, .
    \end{cases}
\end{align}
The Bayes-SE equations derived above serve as the main analytical tool for studying the fixed points and asymptotic performance of the Bayes-AMP algorithm. Their properties are further analyzed in the following sections, including the characterization of phase transitions and recovery regions in the Bayes-AMP framework.

\subsection{Maximum Likelihood Approximate Message Passing}\label{subsec:app_joint_ML_AMP}

In more realistic scenarios where the prior is unknown, the Bayes-optimal framework is not applicable. A natural alternative is to minimize a loss function—in our case, the negative-log-likelihood. It is from this objective that the ML-AMP algorithm is derived by adapting the denoising function. We present below its explicit update equations.
\begin{align}
\hspace{0pt} 
\begin{array}{rl}
\vcenter{\hbox{for $\hat{\vect{u}}$}} &
\left\{
\begin{aligned}
r_1^t &= \frac{\alpha_2}{\Delta_m} \sigma^t_2 + \frac{\alpha_3 \alpha_4}{\rho \Delta_t} \, \para{ \frac{1}{N_3} \langle  \hat{\vect{x}}^t, \hat{\vect{x}}^{t-1} \rangle\sigma^t_4 + \frac{1}{N_4}\langle  \hat{\vect{y}}^t, \hat{\vect{y}}^{t-1} \rangle \sigma^t_3} \, , 
\\[5pt]
\vect{b}^t_1 &= \frac{1}{\Delta_m\sqrt{N_1}} \;\rdmmat{Y}_m \; \hat{\vect{v}}^t  \;
+ \frac{1}{\rho \Delta_t N_1}\;\rdmtens{Y}_{t} \,\cdot\, \para{\hat{\vect{x}}^t \otimes \hat{\vect{y}}^t} \, - \, r_1^t\, \hat{\vect{u}}^{t-1} \, ,
\\[5pt]
\hat{\vect{u}}^{t+1} &= \sqrt{N_1}\frac{\vect{b}^t_1}{\norm{\vect{b}^t_1}} \, , 
\\[5pt]
\sigma_1^{t+1} &= \sqrt{N_1}\frac{1}{\norm{\vect{b}^t_1}} \, .
\end{aligned}
\right.
\end{array}
\end{align}
\begin{align}
\hspace{-115pt} 
\begin{array}{rl}
\vcenter{\hbox{for $\hat{\vect{v}}$}} &
\left\{
\begin{aligned}
r_2^t &= \frac{1}{\Delta_m} \sigma^t_1 \, , 
\\[5pt]
\vect{b}^t_2 &= \frac{1}{\Delta_m\sqrt{N_1}} \;\rdmmat{Y}_m^T \; \hat{\vect{u}}^t \, - \, r_2^t\, \hat{\vect{v}}^{t-1} \, , 
\\[5pt]
\hat{\vect{v}}^{t+1} &= \sqrt{N_2}\frac{\vect{b}^t_2}{\norm{\vect{b}^t_2}} \, , 
\\[5pt]
\sigma_2^{t+1} &= \sqrt{N_2}\frac{1}{\norm{\vect{b}^t_2}} \, .
\end{aligned}
\right.
\end{array}
\end{align}
\begin{align}
\hspace{-45pt} 
\begin{array}{rl}
\vcenter{\hbox{for $\hat{\vect{x}}$}} &
\left\{
\begin{aligned}
r_3^t &= \frac{\alpha_4}{\rho \Delta_t} \, \para{ \frac{1}{N_1} \langle  \hat{\vect{u}}^t, \hat{\vect{u}}^{t-1} \rangle\sigma^t_4 + \frac{1}{N_4}\langle  \hat{\vect{y}}^t, \hat{\vect{y}}^{t-1} \rangle \sigma^t_1} \, , 
\\[5pt]
\vect{b}^t_3 &= \frac{1}{\rho \Delta_t N_1}\;\rdmtens{Y}_{t} \,\cdot\, \para{\hat{\vect{u}}^t \otimes \hat{\vect{y}}^t} \, - \, r_3^t\, \hat{\vect{x}}^{t-1} \, ,
\\[5pt]
\hat{\vect{x}}^{t+1} &= \sqrt{N_3}\frac{\vect{b}^t_3}{\norm{\vect{b}^t_3}} \, , 
\\[5pt]
\sigma_3^{t+1} &= \sqrt{N_3}\frac{1}{\norm{\vect{b}^t_3}} \, .
\end{aligned}
\right.
\end{array}
\end{align}
\begin{align}
\hspace{-45pt} 
\begin{array}{rl}
\vcenter{\hbox{for $\hat{\vect{y}}$}} &
\left\{
\begin{aligned}
r_4^t &= \frac{\alpha_3}{\rho \Delta_t} \, \para{ \frac{1}{N_1} \langle  \hat{\vect{u}}^t, \hat{\vect{u}}^{t-1} \rangle\sigma^t_3 + \frac{1}{N_3}\langle  \hat{\vect{x}}^t, \hat{\vect{x}}^{t-1} \rangle \sigma^t_1} \, , 
\\[5pt]
\vect{b}^t_4 &= \frac{1}{\rho \Delta_t N_1}\;\rdmtens{Y}_{t} \,\cdot\, \para{\hat{\vect{u}}^t \otimes \hat{\vect{x}}^t} \, - \, r_4^t\, \hat{\vect{y}}^{t-1} \, ,
\\[5pt]
\hat{\vect{y}}^{t+1} &= \sqrt{N_4}\frac{\vect{b}^t_4}{\norm{\vect{b}^t_4}} \, .
\\[5pt]
\sigma_4^{t+1} &= \sqrt{N_4}\frac{1}{\norm{\vect{b}^t_4}} \, .
\end{aligned}
\right.
\end{array}
\end{align}

Similarly to the Bayesian case, one can derive a low-dimensional time-recursive equation to predict the performance of ML-AMP that we call ML-SE equations. The rigorous derivation of the generic SE equations follows from \cite{javanmard_state_2013}. However, one can look for a specific derivation of the ML-SE below leading to Eq.~\ref{eq:state_evolution_linearized}. 

In general, one can write:
\begin{align}
    q_1^{t+1} = \frac{1}{N_1}\langle \hat{\vect{u}}^{t+1}, \vect{u}^\star \rangle = \frac{1}{N_1}\langle \vect{\eta}_{\mathrm{ML}}\para{\vect{b}_1^t}, \vect{u}^\star \rangle \, .
\end{align}
However, for the ML case, one can remember the spherical constraint $\norm{\hat{\vect{w}}^t_k}^2/N_k =1$ (or equivalently, $\vect{s}^t=\vect{1}$ with the notations introduced in Eq.~\ref{eq:SE}), it follows:
\begin{align}
    b^t_{1,i} = \mathbb{E}\left[ b^t_{1,i} \right] + \sqrt{\operatorname{Var}\para{b^t_{1,i}}}\,G_i = [\vect{g}_\rho(\vect{q^t})]_1 u_i^\star + \sqrt{[\vect{g}_{\rho^2}(\vect{q^t})]_1} \, G_i \, ,
\end{align}
with the gaussian random vector $G_i \overset{\text{i.i.d.}}{\sim} \mathcal{N}(0,1)$, and we recall $[\vect{g}_\rho(\vect{q^t})]_1 = \frac{\alpha_2}{\Delta_m}q_2^t + \frac{\alpha_3\alpha_4}{\rho \Delta_t}q_3^tq_4^t$. Thus, one can write the closed-form update:
\begin{align}
    q_1^{t+1} = \mathbb{E}_{u_i^\star, G_i}\left[ \eta_{\mathrm{ML}}\para{[\vect{g}_\rho(\vect{q^t})]_1 u_i^\star + \sqrt{[\vect{g}_{\rho^2}(\vect{1})]_1} \, G_i} u_i^\star\right] = \frac{[\vect{g}_\rho(\vect{q^t})]_1}{\sqrt{[\vect{g}_{\rho^2}(\vect{1})]_1+[\vect{g}_\rho(\vect{q^t})]^2_1}} \, .
\end{align}
Repeating this procedure for the remaining variables yields the full ML-SE system:
\begin{align}\label{eq:app_ml_se}
\hspace{-61pt}
    \begin{cases}
        q_1^{t+1} = \frac{\frac{\alpha_2}{\Delta_m}q_2^t +  \frac{\alpha_3\alpha_4}{\rho \Delta_t}q_3^t q_4^t}{\sqrt{\frac{\alpha_2}{\Delta_m} +  \frac{\alpha_3\alpha_4}{\rho^2 \Delta_t} + \para{\frac{\alpha_2}{\Delta_m}q_2^t +  \frac{\alpha_3\alpha_4}{\rho \Delta_t}q_3^t q_4^t}^2}} \, ,
        \\[10pt]
        q_2^{t+1} = \frac{\frac{1}{\Delta_m}q_1^t}{\sqrt{\frac{1}{\Delta_m} + \para{\frac{1}{\Delta_m}q_2^t}^2}} \, ,
        \\[10pt]
        q_3^{t+1} = \frac{\frac{\alpha_4}{\rho \Delta_t}q_1^t q_4^t}{\sqrt{\frac{\alpha_4}{\rho^2 \Delta_t} + \para{\frac{\alpha_4}{\rho \Delta_t}q_1^t q_4^t}^2}} \, ,
        \\[10pt]
        q_4^{t+1} = \frac{\frac{\alpha_3}{\rho \Delta_t}q_1^t q_3^t}{\sqrt{\frac{\alpha_3}{\rho^2 \Delta_t} + \para{\frac{\alpha_3}{\rho \Delta_t}q_1^t q_3^t}^2}} \, .
    \end{cases}
\end{align}
As in the Bayesian setting, this state evolution will serve as the foundation of our analysis. One can observe that the coefficient $\rho$ cancels out in the update equations for $q_3$ and $q_4$. As a result, the only component where the competition between matrix and tensor estimation persists is in the update of $q_1$, associated with the shared signal $\vect{u}$. Notably, even when the tensor is not recovered (i.e., $q_3 = q_4 = 0$), the update of $q_1$ still includes a contribution from the tensor noise. This coupling is precisely what deteriorates the recovery thresholds compared to the matrix-only case. 

\section{Decimating sequential GD to matrix spiked problem}\label{sec:app_equiv_sequential}

In Sec.~\ref{sec:spec}, we analyze a sequential algorithm. In a first time, it tries to optimize the loss $\mathcal{L}_m$ and infer the matricial signals $\vect{u}^\star$ and $\vect{v}^\star$ alone. Then, it uses the estimator of the shared components that we will note $\vect{u}^\textrm{mat}$ in this section and keeps it fixed to infer the tensor components $\vect{x}^\star$ and $\vect{y}^\star$ by minimizing the tensor loss $\mathcal{L}_t$. One can show that the optimization of the second loss reduces to the analysis of a matrix spiked problem. Indeed, it comes:
\begin{align}
    \norm{\rdmtens{Y}_t-\frac{1}{N_1}\vect{u}^{\textrm{mat}} \otimes \vect{x} \otimes \vect{y}}^2_F = \norm{\rdmtens{Y}_t}^2_F - \frac{1}{N_1}\norm{\rdmtens{Y}_t \cdot \vect{u}^{\textrm{mat}}}^2_F + \alpha_4\norm{\frac{1}{\sqrt{N_4}}\rdmtens{Y}_t \cdot \vect{u}^{\textrm{mat}}-\frac{1}{\sqrt{N_4}}\vect{x} \otimes \vect{y}}^2_F \, .
\end{align}
Therefore, optimizing the loss $\mathcal{L}_t$ is equivalent to optimizing this loss $\norm{\frac{1}{\sqrt{N_4}}\rdmtens{Y}_t \cdot \vect{u}^{\textrm{mat}}-\frac{1}{\sqrt{N_4}} \vect{x} \otimes \vect{y}}^2_F$. The latter being an effective matrix spiked problem, it reduces to the optimization of a matrix problem, which is the core of Sec.~\ref{sec:spec}.

This result naturally arises by replacing $\hat{\vect{u}}^t$ with $\vect{u}^{\textrm{mat}}$ in the update equations for $\hat{\vect{x}}^t$ and $\hat{\vect{y}}^t$ in the ML-AMP algorithm defined in Sec.~\ref{subsec:app_joint_ML_AMP} and by setting $\sigma_1^t=\sigma^{\textrm{mat}}_1=0$. The obtained AMP decouples the update of $\hat{\vect{u}}^t$ and $\hat{\vect{v}}^t$ from the updates of $\hat{\vect{x}}^t$ and $\hat{\vect{y}}^t$, and corresponds to a sequential ML-AMP defined separately on the loss $\mathcal{L}_m$ and the loss $\norm{\frac{1}{\sqrt{N_4}}\rdmtens{Y}_t \cdot \vect{u}^{\textrm{mat}}-\frac{1}{\sqrt{N_4}} \vect{x} \otimes \vect{y}}^2_F$.

This result follows naturally by replacing $\hat{\vect{u}}^t$ with $\vect{u}^{\textrm{mat}}$ in the update equations for $\hat{\vect{x}}^t$ and $\hat{\vect{y}}^t$ in the ML-AMP algorithm defined in Sec.~\ref{subsec:app_joint_ML_AMP}, and by setting $\sigma_1^t = \sigma^{\textrm{mat}}_1 = 0$. This modification decouples the updates of $\hat{\vect{u}}^t$ and $\hat{\vect{v}}^t$ from those of $\hat{\vect{x}}^t$ and $\hat{\vect{y}}^t$, effectively yielding a sequential version of ML-AMP. The resulting algorithm optimizes separately the matrix loss $\mathcal{L}_m$ and the tensor loss $\norm{\frac{1}{\sqrt{N_4}}\rdmtens{Y}_t \cdot \vect{u}^{\textrm{mat}}-\frac{1}{\sqrt{N_4}} \vect{x} \otimes \vect{y}}^2_F$. The associated SE equations can be re-derived and give:
\begin{align}\label{eq:app_seq_ml_se}
    \begin{cases}
        q_1^{t+1} = \frac{\frac{\alpha_2}{\Delta_m}q_2^t}{\sqrt{\frac{\alpha_2}{\Delta_m} + \para{\frac{\alpha_2}{\Delta_m}q_2^t}^2}} \, ,
        \\[10pt]
        q_2^{t+1} = \frac{\frac{1}{\Delta_m}q_1^t}{\sqrt{\frac{1}{\Delta_m} + \para{\frac{1}{\Delta_m}q_2^t}^2}} \, ,
    \end{cases} \hspace{1.6cm}
    \begin{cases}
        q_3^{t+1} = \frac{\frac{\alpha_4}{\Delta_t}q_1^{\textrm{mat}} q_4^t}{\sqrt{\frac{\alpha_4}{\Delta_t} + \para{\frac{\alpha_4}{ \Delta_t}q_1^{\textrm{mat}} q_4^t}^2}} \, ,
        \\[10pt]
        q_4^{t+1} = \frac{\frac{\alpha_3}{\Delta_t}q_1^{\textrm{mat}} q_3^t}{\sqrt{\frac{\alpha_3}{\Delta_t} + \para{\frac{\alpha_3}{ \Delta_t}q_1^{\textrm{mat}} q_3^t}^2}} \, ,
    \end{cases}
\end{align}
with $q_1^{\textrm{mat}}$ the overlap of the estimator $\vect{u}^{\textrm{mat}}$. Furthermore, we recall that for spiked matrix models, ML-AMP is theoretically equivalent to PCA, although it is less practical.

\section{Fixed point computation}\label{sec:app_sys}

In order to obtain the fixed points expressions, one has to resolve the SE equations. The fixed points of SE describe the performances of the algorithm. Different regimes appear; while this appendix aims at describing the algorithm performances in a given regime, the App.~\ref{sec:app_stab} yields the thresholds delimiting them. One can note several types of fixed points:
\begin{itemize}
    \item The null fixed point given by $\vect{q}=\vect{0}$.
    \item The matricial fixed point given by $\vect{q}=(q_1,q_2,0,0)$ with $q_1,q_2\neq 0$. 
    \item The matrix and tensor fixed point given by $\vect{q}=(q_1,q_2,q_3,q_4)$ with $q_1,q_2,q_3,q_4\neq0$.
\end{itemize}
Each of those fixed points is related to the regions of the space $\mathcal{R}$, through Ths.~\ref{prop:phase_se_bo}~and~\ref{th:tot_fp_ml}. If the proofs of those theorems characterize these fixed points -- for instance by stating their unicity, stability and else -- this section aims at computing the expressions of the fixed points for a given form of solution. Therefore, we will first introduce a small result, useful for the computation:

\begin{mylemma}
Let $a, b\geq0$ and $\Delta > 0$. Consider the system of equations on $(x_1, x_2) \in [0,1]^2$ given by:
\begin{equation}\label{eq:init_sys}
    \begin{cases}
        x_1 + \dfrac{a}{\Delta} x_1 x_2 = \dfrac{a}{\Delta} x_2 \,, \\[5pt]
        x_2 + \dfrac{b}{\Delta} x_1 x_2 = \dfrac{b}{\Delta} x_1 \,.
    \end{cases}
\end{equation}
Then the unique solution is
\begin{equation}\label{eq:sol_sys}
    \begin{cases}
        x_1 = \dfrac{1}{b} \cdot \dfrac{ab - \Delta^2}{a + \Delta} \,, \\[5pt]
        x_2 = \dfrac{1}{a} \cdot \dfrac{ab - \Delta^2}{b + \Delta} \,.
    \end{cases}
\end{equation}
\end{mylemma}

\begin{proof}
    To solve this system one can consider:
    \begin{align}
        x_+ &= b x_1 + ax_2 \, , \nonumber \\
        x_- &= b x_1 - ax_2 \, ,
    \end{align}
    such that, the Eqs.~\ref{eq:init_sys} become:
    \begin{equation}
        \begin{cases}
        x_+ + \frac{1}{2 \Delta} \para{x_+^2 - x_-^2} = \frac{1}{2 \Delta} \para{a+b}x_+ + \frac{1}{2 \Delta} \para{a-b}x_- \; , \\
        x_-  = \frac{1}{2 \Delta} \para{-a+b}x_+ - \frac{1}{2 \Delta} \para{a+b}x_- \; . 
        \end{cases}
    \end{equation}
    Then, one can directly obtain the aimed result.
\end{proof}

This result gives the solutions of SE in different settings for the matricial fixed points. The matrix and tensor cannot be easily computed and also are not unique. A more detailed analysis -- also using this result -- of these peculiar fixed points is given in App.~\ref{sec:app_hard_phase}. However, this result may be useful for the analysis of such fixed points. Below, we provide an exhaustive list of how this result can be used to obtain the fixed point values.

\begin{itemize}
    \item  \underline{Bayes matrix fixed point:}
    
        We are looking for fixed points of the Bayes-SE equations with the condition $q_3=q_4=0$ and $q_1,q_2\neq0$. Thus, one can turn the update on $q_1$ and $q_2$ into:
        \begin{align}
            \begin{cases}
                q_1 + \frac{\alpha_2}{\Delta_m}q_1 q_2 = \frac{\alpha_2}{\Delta_m} q_2 \\[5pt]
                q_2 + \frac{1}{\Delta_m}q_1 q_2= \frac{1}{\Delta_m} q_1 
            \end{cases}
        \end{align}
        Then, one can use the aforementioned result and obtains, the expression of the Bayes matrix fixed point written in Eq.~\ref{eq:qu_mat_bo}.

    \item \underline{ML matrix fixed point:}

    Similarly as before, we are looking for a fixed point of ML-SE equations with $q_3=q_4=0$ and $q_1,q_2\neq0$. One can turn the ML-SE equations in the following form:
    \begin{align}
        \begin{cases}
            q^2_1 + \frac{\tilde{\alpha}_2}{\Delta_m}q^2_1 q^2_2 = \frac{\tilde{\alpha}_2}{\Delta_m} q^2_2 \\[5pt]
            q^2_2 + \frac{1}{\Delta_m}q^2_1 q^2_2 = \frac{1}{\Delta_m} q^2_1 
        \end{cases}
    \end{align}
    Then, applying the result leads to Eqs.~\ref{eq:qu_mat_ml}.

    \item  \underline{Bayes matrix and tensor fixed points:}

    We are looking for fixed points of the Bayes-SE equations with the condition $q_1,q_2,q_3,q_4\neq0$. Thus, one can turn the update on $q_3$ and $q_4$ into:
    \begin{align}
        \begin{cases}
            q_3 + \frac{\alpha_4}{\Delta_t}q_1 q_3 q_4 = \frac{\alpha_4}{\Delta_t} q_1 q_4 \\[5pt]
            q_4 + \frac{\alpha_3}{\Delta_t}q_1 q_3 q_4 = \frac{\alpha_3}{\Delta_t} q_1 q_3 
        \end{cases}
    \end{align}
    Then, one can use the aforementioned result and obtains, the expression of the overlaps $q_3$ and $q_4$ in terms of $q_1$ and the parameters of the system. This result will be useful to analyze the Hard Phase of the Bayes-AMP algorithm, developed in Sec.~\ref{sec:app_hard_phase}.

    \item \underline{ML matrix and tensor fixed points:}

    Similarly as before, we are looking for a fixed point of ML-SE equations with $q_1,q_2,q_3,q_4\neq0$. One can turn the ML-SE equations in the following form:
    \begin{align}
        \begin{cases}
            q^2_3 + \frac{\alpha_4}{\Delta_t}q^2_1 q^2_3 q^2_4 = \frac{\alpha_4}{\Delta_t} q^2_1 q^2_4 \\[5pt]
            q^2_4 + \frac{\alpha_3}{\Delta_t}q^2_1 q^2_3 q^2_4 = \frac{\alpha_3}{\Delta_t} q^2_1 q^2_3  
        \end{cases}
    \end{align}
    Similarly to the Bayesian case, the result yields the expressions of $q_3$ and $q_4$ as a function of $q_1$ and the parameter of the system. This result will be crucial for further analysis of the ML-AMP algorithm in Sec.~\ref{sec:app_hard_phase}.

    \item \underline{Sequential fixed points:}

    In Sec.~\ref{sec:spec}, we considered an algorithm based on two successive spectral methods. This approach leads to the fixed points described in Eqs.~\ref{eq:overlaps_12_seq}~and~\ref{eq:overlaps_34_seq}, derived with rigorous arguments from random matrix theory. In App.~\ref{sec:app_equiv_sequential}, we showed that this spectral method is in fact equivalent to a sequential ML-AMP algorithm, and we provided the corresponding state evolution equations. The fixed points of the SE equations reduce to the following system:
    \begin{align}
        \begin{cases}
            q^2_1 + \frac{\alpha_2}{\Delta_m}q^2_1 q^2_2 = \frac{\alpha_2}{\Delta_m} q^2_2 \\[5pt]
            q^2_2 + \frac{1}{\Delta_m}q^2_1 q^2_2 = \frac{1}{\Delta_m} q^2_1 
        \end{cases}\hspace{1cm}
        \begin{cases}
            q^2_3 + \frac{\alpha_4}{\Delta_t}q^2_1 q^2_3 q^2_4 = \frac{\alpha_4}{\Delta_t} q^2_1 q^2_4 \\[5pt]
            q^2_4 + \frac{\alpha_3}{\Delta_t}q^2_1 q^2_3 q^2_4 = \frac{\alpha_3}{\Delta_t} q^2_1 q^2_3  
        \end{cases}
    \end{align}
    Thus, the aforementioned result directly yields the expression of $\vect{q}^s$. 
\end{itemize}

\section{Stability Analysis}\label{sec:app_stab}

In this Appendix, we derive the calculation of the thresholds delimiting the Bayesian recovery regions ($\mathcal{R}_0$, $\mathcal{R}_m$, $\mathcal{R}_t$) defined in Eq.~\ref{eq:region_0}~to~\ref{eq:region_t} and the effective ones delimiting the Joint ML recovery regions ($\tilde{\mathcal{R}}_0$, $\tilde{\mathcal{R}}_m$, $\tilde{\mathcal{R}}_t$). We note a fixed point of the SE equations: $\vect{q}^0=(q^0_1, q^0_2, q^0_3, q^0_4)$. According to Eqs.\ref{eq:state_evolution_BO}~and~\ref{eq:state_evolution_linearized}, the fixed point of the SE equations can be written in a compact form:
\begin{equation}
    \vect{q}^0 = \vect{f}\para{\vect{g} \para{\vect{q}^0}} \, .
\end{equation}
The function $\vect{g}$ is defined in Eq.~\ref{eq:g_entry_wise}. Namely,
\begin{equation}
    \begin{cases}
    \left[\vect{g}\para{\vect{q}^{t}}\right]_1 = \frac{\alpha_2}{\Delta_m}q^t_2 + \frac{\alpha_3 \alpha_4}{\rho \Delta_t}q^t_3 q^t_4 \; , \\[5pt] 
    \left[\vect{g}\para{\vect{q}^{t}}\right]_2 = \frac{1}{\Delta_m} q^t_1 \; , \\[5pt]
    \left[\vect{g}\para{\vect{q}^{t}}\right]_3 = \frac{\alpha_4}{\rho \Delta_t}q^t_1 q^{t}_4 \; ,  \\[5pt] 
    \left[\vect{g}\para{\vect{q}^{t}}\right]_4 = \frac{\alpha_3}{\rho \Delta_t} q_1^t q_3^t \; .
    \end{cases}
\end{equation}
Additionally, exhaustive versions of SE are given in App.~\ref{sec:app_algo}.
We are aiming at analyzing the stability of this fixed point, to do so, let us consider a perturbation $\delta \vect{q}^t$ around the fixed point $\vect{q}^0$. We define the perturbed fixed point $\vect{q}^t = \vect{q}^0 + \delta \vect{q}^t$. Thus, the SE equations can be written as:
\begin{equation}
    \delta \vect{q}^{t+1} =  \nabla\vect{f}\para{\vect{g} (\vect{q}^0)} \,\odot\, \delta \para{\vect{g}\para{\vect{q}^{t}}}
\end{equation}
Thus each of the two terms of the right-hand side can be computed apart:
\begin{align}
    \nabla \vect{f}_\textrm{Bayes}(\vect{g}(\vect{q}^0)) &= \para{\para{1+\left[\vect{g}_{\rho=1}(\vect{q}^0)\right]_k}^{-2}}_{k \in \llbracket 4 \rrbracket} \; ,  \nonumber \\[5pt]
    \nabla \vect{f}_\textrm{ML}(\vect{g}(\vect{q}^0)) &= \para{[\vect{g}_{\rho^2}(\vect{1})]_k\para{\left[\vect{g}_{\rho^2}(\vect{1})\right]_k+ \left[\vect{g}_\rho(\vect{q}^0)\right]_k^2}^{-3/2}}_{k \in \llbracket 4 \rrbracket} \; .
\end{align}
and
\begin{equation}
    \begin{cases}
    \left[\delta \para{\vect{g}\para{\vect{q}^{t}}}\right]_1 = \frac{\alpha_2}{\Delta_m}\delta q^t_2 + \frac{\alpha_3 \alpha_4}{\rho \Delta_t}\para{q^0_3 \delta q^t_4 + q^0_4 \delta q^t_3} \; , \\[5pt] 
    \left[\delta \para{\vect{g}\para{\vect{q}^{t}}}\right]_2 = \frac{1}{\Delta_m}\delta q^t_1 \; , \\[5pt]
    \left[\delta \para{\vect{g}\para{\vect{q}^{t}}}\right]_3 = \frac{\alpha_4}{\rho \Delta_t}\para{q^0_1 \delta q^t_4 + q^0_4 \delta q^t_1} \; ,  \\[5pt]
    \left[\delta \para{\vect{g}\para{\vect{q}^{t}}}\right]_4 = \frac{\alpha_3}{\rho \Delta_t}\para{q^0_1 \delta q^{t}_3 + q^0_3 \delta q^t_1} \; .
    \end{cases}
\end{equation}

Furthermore, it is possible to simplify even more these equations for specific fixed points. 

\subsection{Stability of the null fixed point}\label{subsec:app_null_fp}

If we focus on the null fixed point known as the fixed point $\vect{q}^0 = 0$, it comes:
\begin{equation} \label{eq:perturb_evol_null}
    \begin{cases}
    \delta q^{t+1}_1 = \frac{\alpha_2}{\Delta_m} \left[\nabla\vect{f}(\vect{0})\right]_1 \delta q^t_2 \; , \\[5pt]
    \delta q^{t+1}_2 = \frac{1}{\Delta_m} \left[\nabla\vect{f}(\vect{0})\right]_2 \delta q^t_1  \; , \\[5pt]
    \delta q^{t+1}_3 = 0 \; ,  \\[5pt]
    \delta q^{t+1}_4 = 0 \; .
    \end{cases}
\end{equation}
Thus, the stability of the null fixed point is given by the following condition:
\begin{equation}\label{eq:cond_stab_null}
    \frac{\alpha_2}{\Delta_m^2} \left[\nabla\vect{f}(\vect{0})\right]_1 \left[\nabla\vect{f}(\vect{0})\right]_2  < 1 \, .
\end{equation}
This threshold defines the regions $\mathcal{R}_0$ and $\tilde{\mathcal{R}}_0$. For instance: 
\begin{itemize}
    \item \underline{Bayes null fixed point:}
        \begin{align}
            \Delta_m>\sqrt{\alpha_2}
        \end{align}
    \item \underline{ML null fixed point:}
        \begin{align}
            \Delta_m>\sqrt{\alpha_2} \para{\frac{\frac{\alpha_2}{\Delta_m}}{\frac{\alpha_2}{\Delta_m} + \frac{\alpha_3 \alpha_4}{\rho^2\Delta_t}}}^{1/2} = \sqrt{\tilde{\alpha}_2}
        \end{align}
    \item \underline{Sequential fixed points:}

        Recalling the mapping between the sequential ML-AMP algorithm and the spectral method detailed in App.~\ref{sec:app_equiv_sequential}, one can apply the fixed point analysis successively: first on the matrix estimation, then on the tensor estimation conditioned on the matrix output. For the null fixed point to be stable in both stages, the following condition must hold:
        \begin{align}
            \Delta_m>\sqrt{\alpha_2}
        \end{align}
        On the other hand, the existence and stability of a matrix fixed point—which corresponds to the null fixed point in the tensor estimation—require, respectively:
        \begin{align}
            \Delta_m < \sqrt{\alpha_2} \quad \textrm{and} \quad \Delta_t > \alpha_3 \alpha_4 \frac{\alpha_2-\Delta_m^2}{\alpha_2+\Delta_m} = \delta_c^{\mathrm{Bayes}}(\alpha_2, \alpha_3, \alpha_4, \Delta_m) \, .
        \end{align}
\end{itemize}

\subsection{Stability of the matrix fixed point}\label{subsec:app_matrix_fp}

The matrix fixed point, corresponding to the situation where we can recover only the matrix, has overlaps $q_3=q_4=0$ and $q_1, q_2 \neq 0$. Then, in this configuration, the stability system can be written as:
\begin{equation} \label{eq:perturb_evol_mat}
    \begin{cases}
    \delta q^{t+1}_1 = \frac{\alpha_2}{\Delta_m} \left[\nabla\vect{f}(\vect{g}(\vect{q}^0))\right]_1 \delta q^t_2 \; , \\[5pt] 
    \delta q^{t+1}_2 = \frac{1}{\Delta_m} \left[\nabla\vect{f}(\vect{g}(\vect{q}^0))\right]_2 \delta q^t_1 \; , \\[5pt]
    \delta q^{t+1}_3 = \frac{\alpha_4}{\rho \Delta_t} q_1^0 \left[\nabla\vect{f}(\vect{0})\right]_3 \delta q^t_4  \; ,  \\[5pt]
    \delta q^{t+1}_4 = \frac{\alpha_3}{\rho \Delta_t} q_1^0 \left[\nabla\vect{f}(\vect{0})\right]_4 \delta q^t_3 \; .
    \end{cases}
\end{equation}
Thus, the stability condition can be written as:
\begin{equation}\label{eq:cond_stab_mat}
    \frac{\alpha_2}{\Delta_m^2} \, \left[\nabla\vect{f}(\vect{g}(\vect{q}^0))\right]_1\, \left[\nabla\vect{f}(\vect{g}(\vect{q}^0))\right]_2 < 1 \quad \textrm{and} \quad \frac{\alpha_3 \alpha_4}{\Delta_t^2}\, \big(q_1^0\big)^2\,  \left[\nabla\vect{f}(\vect{0})\right]_3 \left[\nabla\vect{f}(\vect{0})\right]_4 < 1
\end{equation}
This threshold defines the regions $\mathcal{R}_m$ and $\tilde{\mathcal{R}}_m$. Specifically,:
\begin{itemize}
    \item \underline{Bayes matrix fixed point:}
        \begin{align}
            \Delta_m < \sqrt{\alpha_2} \quad \textrm{and} \quad \Delta_t > \alpha_3 \alpha_4 \frac{\alpha_2-\Delta_m^2}{\alpha_2+\Delta_m} = \delta_c^{\mathrm{Bayes}}(\alpha_2, \alpha_3, \alpha_4, \Delta_m) \, , 
        \end{align}
    \item \underline{ML matrix fixed point:}
        \begin{align}
            \Delta_m < \sqrt{\tilde{\alpha}_2} \quad \textrm{and} \quad \Delta_t > \alpha_3 \alpha_4 \frac{\tilde{\alpha}_2-\Delta_m^2}{\tilde{\alpha}_2+\Delta_m} = \delta_c^{\mathrm{Bayes}}(\tilde{\alpha}_2, \alpha_3, \alpha_4, \Delta_m) = \tilde{\delta}_c \, .
        \end{align}
\end{itemize}

\section{Free energy of the model}\label{sec:free_e}

In a Bayesian setting, the posterior distribution over the signals $\vect{u}$, $\vect{v}$, $\vect{x}$ and $\vect{y}$, given the observations $\rdmmat{Y}_m, \rdmtens{Y}_t$, is defined in Eq.~\ref{eq:post}, and involves the normalization constant known as the \emph{partition function}~$\mathcal{Z}$. The associated free energy is defined by:
\begin{align}
    \Phi = \lim_{N_1\to\infty} \mathbb{E}_{\rdmmat{Y}_m, \rdmtens{Y}_t}\left[- \frac{1}{N_1} \log \mathcal{Z}\para{\rdmmat{Y}_m, \rdmtens{Y}_t}\right] \, .
\end{align}
Computing this quantity directly is intractable. Several non-rigorous approaches can be used, such as the \emph{replica method} or derivations based on the \emph{Bethe Free Energy} through the AMP framework. However, a rigorous expression for the free energy can be adapted from closely related settings where it has been proven: namely, the non-symmetric order-3 tensor spiked model~\cite{barbier2017layered}, and the symmetric matrix-tensor case~\cite{Sarao_Mannelli_marvels}. In our case, this leads to:
\begin{align}
    &\Phi = \min_{\vect{q}} \tilde{\Phi}(\vect{q}) \, ,\qquad \mathrm{where} \\
    &\tilde{\Phi}(\vect{q}) = \frac{\alpha_2}{2 \Delta_m} q_1 q_2 + \frac{\alpha_3 \alpha_4}{2 \Delta_t} q_1 q_3 q_4 + \frac{1}{2}\sum_{k=1}^4 \alpha_k \left[ q_k + \log \left(1 - q_k \right) \right] \, ,
\end{align}
with $\alpha_1 = 1$ by convention. 

Finally, the Bayes-optimal State Evolution (Bayes-SE) fixed points are characterized as the critical points of the potential $\tilde{\Phi}$. In particular, they are given by:
\begin{equation}
    \nabla_{\vect{q}} \Phi(\vect{q}) = \vect{0} \, ,
\end{equation}
as recalled in Eqs.~\ref{eq:app_bo_se}.

\section{Spinodal and Hard Phase}\label{sec:app_hard_phase}

In Secs.~\ref{sec:bo_amp}~and~\ref{sec:ml_amp}, we analyzed the Bayes-AMP and the ML-AMP algorithms. However, the phase diagrams discussed in these sections omitted a peculiar aspect: the hard phase and, more generally, the dependency on the initialization. In this appendix, we detail the methods used to characterize this dependency in both Bayes and ML settings. 

\subsection{Hard phase analysis in the Bayesian setting}

\begin{figure}[t]
    \centering
    \begin{subfigure}{\textwidth}
        \centering
        \includegraphics[width=0.48\textwidth]{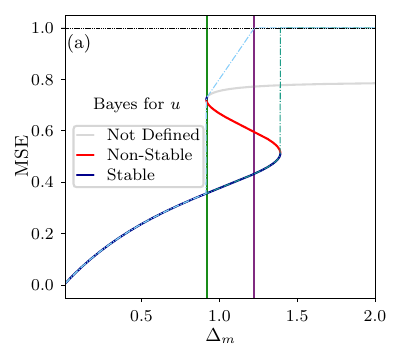}
        \hspace{0.01cm}
        \includegraphics[width=0.48\textwidth]{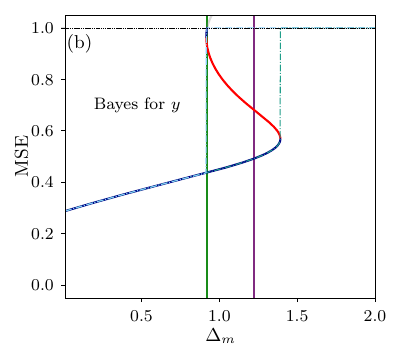}
    \end{subfigure}
    \caption{Mean Square Error (MSE) as a function of $\Delta_m$ for the Bayes-AMP algorithm at fixed $\Delta_t = 0.24$, with $\alpha_2 = 1.5$, $\alpha_3 = 0.8$, and $\alpha_4 = 1$. Panel (a) shows the MSE for the $u$ estimator, while panel (b) shows the MSE for the $y$ estimator. Stable, unstable, and non-defined fixed points are respectively shown in blue, red, and gray. The green and purple vertical lines correspond to the phase transitions $\Delta_m = \sqrt{\alpha_2}$ and $\Delta_t = \delta_c^{\mathrm{Bayes}}$. The blue and green dash-dot lines correspond to the iterative Bayes-SE initialized respectively with an uninformative and an informative overlap vector: $\vect{q}^{t=0} = 0.05\,\vect{1}$ and $\vect{q}^{t=0} = 0.9\,\vect{1}$.}
    \label{fig:mse_bayes}
\end{figure}

\begin{figure}[t]
    \centering
    \begin{subfigure}{\textwidth}
        \centering
        \includegraphics[width=0.48\textwidth]{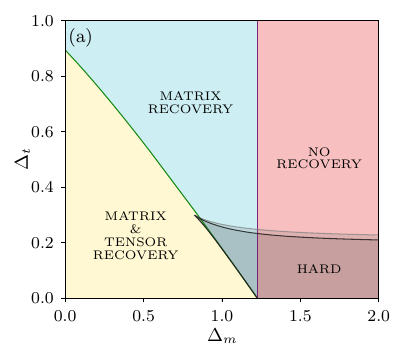}
        \hspace{0.01cm}
        \includegraphics[width=0.48\textwidth]{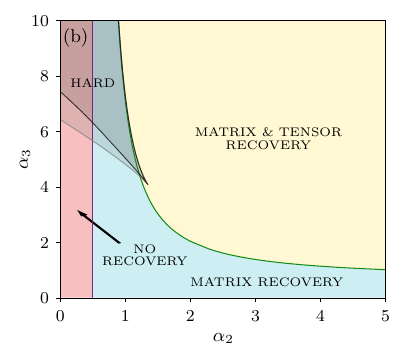}
    \end{subfigure}
    \caption{Phase diagrams for Bayes-AMP algorithm, (a) in $\para{\Delta_m, \Delta_t}$ space  with size ratios $\alpha_2 = 1.5$, $\alpha_3 = 0.8$, and $\alpha_4 = 1$ and in $\para{\alpha_2, \alpha_3}$ space (b) in $\para{\alpha_2, \alpha_3}$ space with noise $\Delta_m = 0.7$, $\Delta_t=0.8$ and a size ratio $\alpha_4=1$. The red, blue and yellow regions respectively delimit $\mathcal{R}_0$, $\mathcal{R}_m$ and $\mathcal{R}_t$. The purple and green solid lines indicate the phase transitions between the regions $\mathcal{R}_0$ and $\mathcal{R}_m$, and between $\mathcal{R}_m$ and $\mathcal{R}_t$, respectively. These transitions are implicitly defined by $\Delta_m = \sqrt{\alpha_2}$ and $\Delta_t = \delta_c^{\mathrm{Bayes}}$. The black and gray lines denote the boundaries of the hard phase and the spinodal transitions, respectively, with their corresponding regions shaded in black and gray.}
    \label{fig:hard_bayes}
\end{figure}

In this appendix, we provide a more refined analysis of the matrix-tensor fixed points, in which all overlaps are nonzero. We show in Sec.~\ref{sec:app_sys} that for Bayes-AMP, the tensor overlaps $q_3$ and $q_4$ can be expressed as functions of the matrix overlap $q_1$. Then, one can obtain the following system of equations:
\begin{align}
    \begin{cases}
        q_1 + \para{\frac{\alpha_2}{\Delta_m} q_2 + \frac{\alpha_3 \alpha_4} {\Delta_t} q_3 q_4 }\para{q_1 -1} = 0 \; , \\[5pt]
        q_2 = \frac{q_1}{\Delta_m+q_1}\; , \\[5pt]
        q_3 = \frac{1}{\alpha_3 q_1 } \frac{\alpha_3 \alpha_4 q_1^2 - \Delta_t^2}{\alpha_4 q_1 +\Delta_t} \; , \\[5pt]
        q_4 = \frac{1}{\alpha_4 q_1 } \frac{\alpha_3 \alpha_4 q_1^2 - \Delta_t^2}{\alpha_3 q_1 + \Delta_t} \; .
    \end{cases}
\end{align}
This system reduces the SE to a degree-6 polynomial in $q_1$ in the Bayesian setting. These polynomial forms are numerically convenient, as they reveal all fixed points—including unstable ones—with reasonable computational cost. Figure~\ref{fig:mse_bayes} illustrates the evolution of these fixed points and the associated transitions across different regimes.

The roots of the resulting polynomial describe the informative, uninformative, and unstable fixed points of the system. For completeness, we also include roots that fall outside the domain $[0,1]^4$, which exhibit the familiar structure reported in~\cite{Lesieur_2017}. Informative (resp. uninformative) fixed points are defined as those reached by informative (resp. uninformative) initialization.

Spinodal transitions are identified by locating the parameter values at which the polynomial roots acquire a vertical tangent. These critical points delineate the regions of spinodal existence, shown in Fig.~\ref{fig:hard_bayes}, and indicate where informed initialization significantly improves recovery performance.

Fig.~\ref{fig:mse_bayes} can be interpreted as a slice of the phase diagram in Fig.~\ref{fig:hard_bayes} at fixed $\Delta_t = 0.24$. As $\Delta_m$ increases, we observe the following regimes. Note that we only describe here the phases that appear in this particular plot; additional regimes, present in the full phase diagram, are not visible in this slice.
\begin{enumerate}[label=(\roman*)]
    \item In region $\mathcal{R}_t$, only one stable matrix-tensor fixed point exists, which is uninformative. Bayes-AMP converges to it regardless of initialization.
    \item In the overlap between $\mathcal{R}_t$ and the spinodal region, three fixed points appear: an uninformative stable one, an informative stable one, and an unstable one. This confirms that the tri-critical point does not coincide with the boundary of $\mathcal{R}_t$.
    \item In region $\mathcal{R}_m$, the stable uninformative fixed point becomes purely matricial (denoted $\vect{q}_m$), while a stable informative matrix-tensor fixed point and an unstable one coexist.
    \item In region $\mathcal{R}_0$, the only stable point reached from an uninformative initialization is the trivial (null) one. However, an informative initialization can still lead to recovery through a matrix-tensor fixed point.
    \item Finally, beyond a certain threshold, only the null fixed point remains.
\end{enumerate}

The hardness of Bayes-AMP can be quantified using the free energy from App.~\ref{sec:free_e}. In regions where both informative and uninformative fixed points exist, a hard phase emerges when the global minimizer of the free energy corresponds to the informative fixed point. Such a solution can only be reached with an informed initialization. Comparing the respective free energies computed via the expressions provided above allows one to identify this hard phase precisely.

\subsection{Further analysis in the ML setting}

\begin{figure}[t]
    \centering
    \begin{subfigure}{\textwidth}
        \centering
        \includegraphics[width=0.48\textwidth]{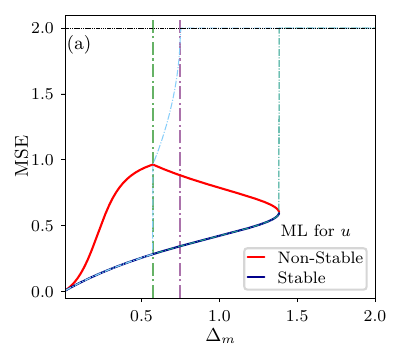}
        \hspace{0.01cm}
        \includegraphics[width=0.48\textwidth]{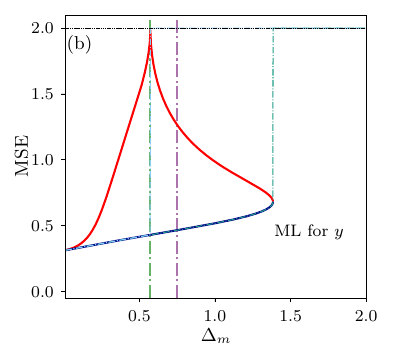}
    \end{subfigure}
    \caption{Mean Square Error (MSE) as a function of $\Delta_m$ for the ML-AMP algorithm at fixed $\Delta_t = 0.24$, with $\alpha_2 = 1.5$, $\alpha_3 = 0.8$, $\alpha_4 = 1$, and $\rho=1$. Panel (a) shows the MSE for the $u$ estimator, while panel (b) shows the MSE for the $y$ estimator. Stable and unstable fixed points are respectively shown in blue and red. The green and purple dash-dot lines correspond to the phase transitions $\Delta_m = \sqrt{\tilde{\alpha}_2}$ and $\Delta_t = \tilde{\delta}_c$. The blue and green dash-dot lines correspond to the iterative ML-SE initialized respectively with an uninformative and an informative overlap vector: $\vect{q}^{t=0} = 0.2\,\vect{1}$ and $\vect{q}^{t=0} = 0.9\,\vect{1}$.  }
    \label{fig:mse_ml}
\end{figure}

The derivation of the matrix-tensor fixed points in the ML estimation setting follows the same structure as in the Bayesian case. Using the result from App.~\ref{sec:app_sys}, the ML-SE equations can be reformulated into a polynomial equation of degree 12 in $q_1^2$:
\begin{align}
    \begin{cases}
        \para{\frac{\alpha_2}{\Delta_m} + \frac{\alpha_3\alpha_4}{\rho^2 \Delta_t}} q^2_1 + \para{\frac{\alpha_2}{\Delta_m}  q_2 + \frac{\alpha_3 \alpha_4}{\rho \Delta_t} q_3 q_4}^2 \para{q^2_1-1}  = 0   \, , \\[5pt]
        q^2_2 = \frac{q^2_1}{\Delta_m+q^2_1}\; , \\[5pt]
        q^2_3 = \frac{1}{\alpha_3 q^2_1 } \frac{\alpha_3 \alpha_4 q_1^4 - \Delta_t^2}{\alpha_4 q^2_1 +\Delta_t} \; , \\[5pt]
        q^2_4 = \frac{1}{\alpha_4 q^2_1 } \frac{\alpha_3 \alpha_4 q_1^4 - \Delta_t^2}{\alpha_3 q^2_1 + \Delta_t} \; .
    \end{cases}
\end{align}
This formulation allows for an efficient numerical computation of both stable and unstable matrix-tensor fixed points. These are illustrated in Fig.~\ref{fig:mse_ml}, where we compare ML-SE trajectories under informative and uninformative initializations. The transitions observed for the uninformative ML-SE match closely the analytically derived thresholds. In contrast, the informative ML-SE achieves better performance; this behavior mirrors the one observed in the Bayesian setting. This highlights the crucial role of initialization, even in the ML-AMP case. In particular, it is possible to recover both matrix and tensor for certain parameter regimes where the uninformative ML-AMP fails entirely, but where a sequential initialization leads to successful recovery—despite worse performance than the informative ML-AMP. This suggests a natural algorithmic strategy: initialize ML-AMP using the sequential spectral method, and optimize the parameter $\rho$ to further improve performance.

\section{Further numerical results}

In this section, we provide numerical results for the two main algorithms Bayes-AMP and ML-AMP in Fig.~\ref{fig:num_bayes}~and~\ref{fig:num_ml} respectively.

\begin{figure}[ht]
    \centering
    \includegraphics[width=0.8\textwidth]{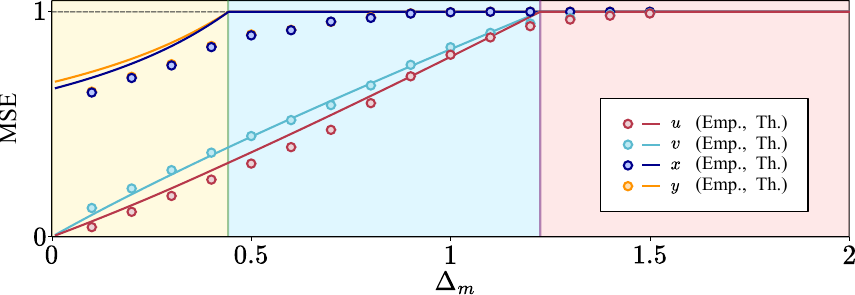}
    \caption{Empirical and theoretical MSE vs. $\Delta_m$ in the Bayesian case, with parameters $\Delta_t=0.6,\; N_1=10^3, \; N_2 = 1.5\,.\,10^3,\; N_3 = 0.8\,.\,10^3,\;N_4 = 10^3$. The empirical values (dots) are computed numerically via Bayes-AMP algorithm initialized with an SVD on the matrix, averaged over 50 samples. The error bars are not visible on the plots due to their small magnitude. Their theoretical predictions (solid lines) are given by the iterative Bayes-SE algorithm. The background colors match the three regions $\mathcal{R}_0$, $\mathcal{R}_m$, and $\mathcal{R}_t$. The purple and green vertical lines indicate the phase transitions between the regions $\mathcal{R}_0$ and $\mathcal{R}_m$, and between $\mathcal{R}_m$ and $\mathcal{R}_t$, respectively. These transitions are implicitly defined by $\Delta_m = \sqrt{\alpha_2}$ and $\Delta_t = \delta_c^{\mathrm{Bayes}}$.
    }
    \label{fig:num_bayes}
\end{figure}

\begin{figure}[ht]
    \centering
    \includegraphics[width=0.8\textwidth]{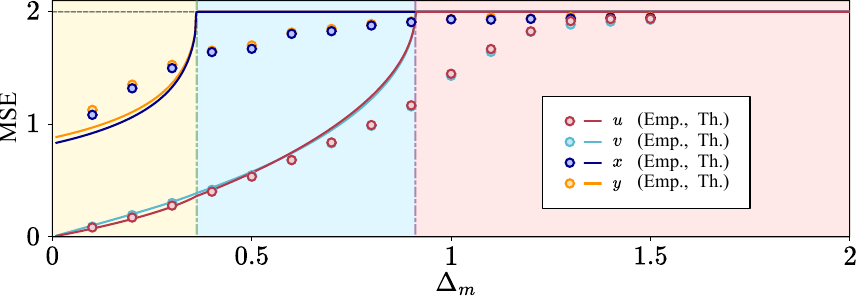}
    \caption{Empirical and theoretical MSE vs. $\Delta_m$ in the ML case, with parameters $\Delta_t=0.6,\; N_1=10^3, \; N_2 = 1.5\,.\,10^3,\; N_3 = 0.8\,.\,10^3,\;N_4 = 10^3$, and $\rho=1$. The empirical values (dots) are computed numerically via ML-AMP algorithm initialized with an SVD on the matrix, averaged over 50 samples. The error bars are not visible on the plots due to their small magnitude. Their theoretical predictions (solid lines) are given by the iterative ML-SE algorithm. The background colors match the three regions $\tilde{\mathcal{R}}_0$, $\tilde{\mathcal{R}}_m$, and $\tilde{\mathcal{R}}_t$. The purple and green dash-dot lines indicate the phase transitions between the regions $\tilde{\mathcal{R}}_0$ and $\tilde{\mathcal{R}}_m$, and between $\tilde{\mathcal{R}}_m$ and $\tilde{\mathcal{R}}_t$, respectively. These transitions are implicitly defined by $\Delta_m = \sqrt{\tilde{\alpha}_2}$ and $\Delta_t = \tilde{\delta}_c$.
    }
    \label{fig:num_ml}
\end{figure}

\end{document}